\pdfoutput=1
\def\djepaauthorversion{1}
\documentclass{article}
\usepackage{iclr2027_conference,times}
\usepackage[T1]{fontenc}
\usepackage[utf8]{inputenc}
\usepackage{microtype,graphicx,amsmath,amssymb,amsthm}
\usepackage[table]{xcolor}
\usepackage{booktabs,tabularx,array,makecell}
\usepackage{caption,subcaption}
\usepackage{tikz}
\usepackage{enumitem,placeins}
\usepackage{etoolbox}
\usepackage{hyperref}
\definecolor{headerbg}{HTML}{F3EDF6}
\definecolor{groupbg}{HTML}{FAF7FC}
\definecolor{sectiongray}{HTML}{F2F2F2}
\definecolor{variantbg}{HTML}{F7F2F9}
\definecolor{oursbg}{HTML}{EFE4F4}
\definecolor{referencebg}{HTML}{FFF8E7}
\definecolor{badgepurple}{HTML}{E7D8ED}
\definecolor{badgeyellow}{HTML}{FCEFCF}
\definecolor{methodpurple}{HTML}{A45AA8}
\definecolor{ink}{HTML}{3D3840}
\newcommand{\titlewithicon}[2]{%
  \IfFileExists{figures/jepa-icon-light.pdf}{%
    \parbox[c]{40pt}{\includegraphics[width=40pt]{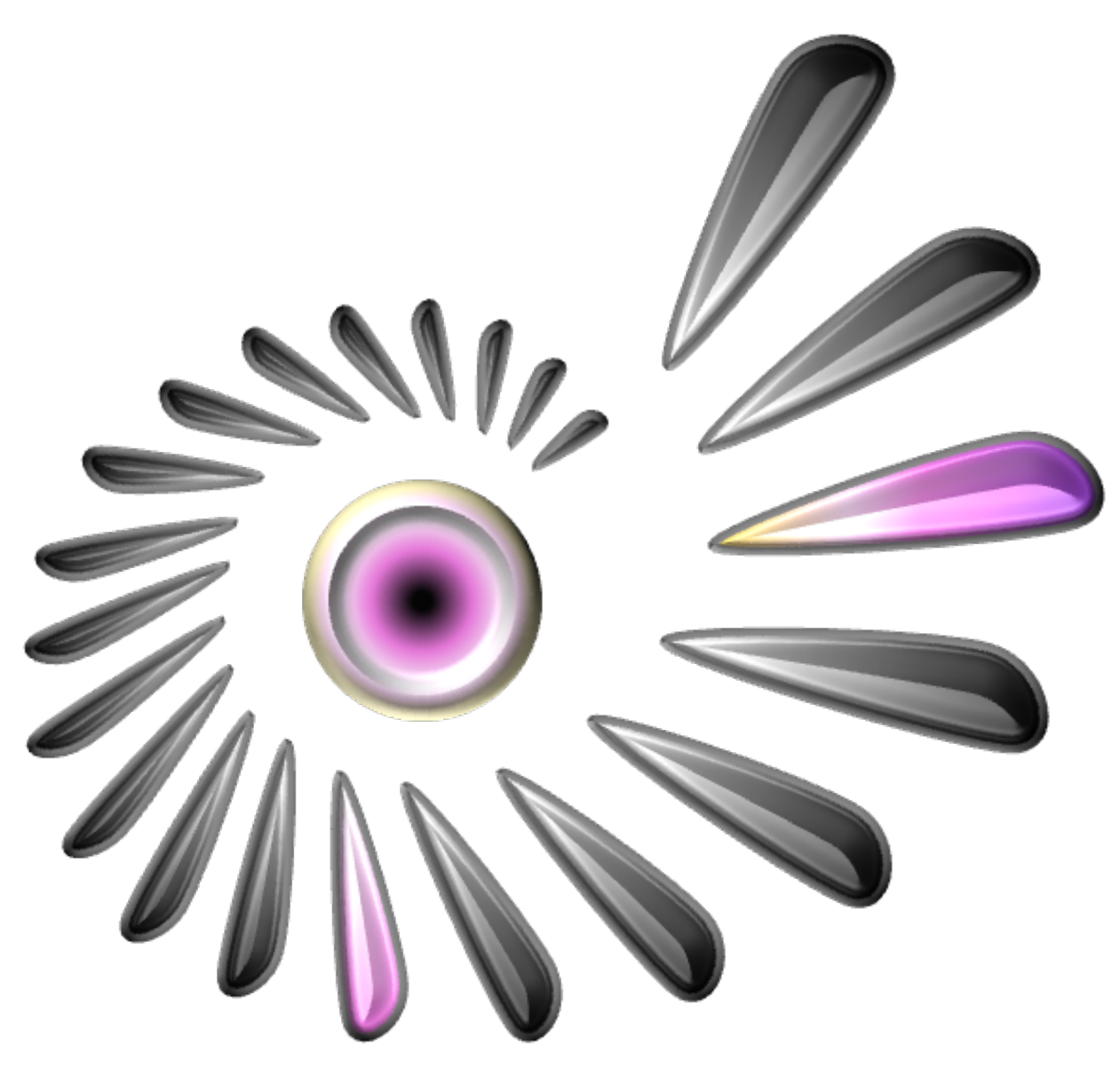}}%
    \hspace{10pt}%
    \parbox[c]{\dimexpr\textwidth-50pt\relax}{\raggedright #1\\#2}%
  }{#1\\#2}}
\newcommand{\method}{D-JEPA}
\newcommand{\bestcell}[1]{\textbf{#1}}
\newcommand{\pubref}[3]{\hyperlink{cite.#3}{(#1, \textit{#2})}\nocite{#3}}
\newcommand{\yearref}[3]{\hyperlink{cite.#3}{(#1, #2)}\nocite{#3}}
\makeatletter
\newsavebox{\paperbadgebox}
\newcommand{\paperbadge}[2][badgepurple]{%
  \begingroup
  \sbox{\paperbadgebox}{\fontsize{5.1}{6.0}\selectfont\sffamily\mdseries #2}%
  \raisebox{\dimexpr(\ht\@arstrutbox-\dp\@arstrutbox)/2\relax}{%
    \tikz[baseline=(badge.center)]{\node[rounded corners=1.8pt,
      fill=#1,draw=#1!65!black,line width=.22pt,
      inner xsep=2pt,inner ysep=.65pt,text=ink] (badge) {\usebox{\paperbadgebox}};}}%
  \endgroup}
\makeatother

\newcommand{\rank}{\operatorname{rank}}
\newcommand{\softplus}{\operatorname{softplus}}
\newcommand{\argmin}{\operatorname*{arg\,min}}
\newcommand{\argmax}{\operatorname*{arg\,max}}
\newcolumntype{L}{>{\raggedright\arraybackslash}X}
\newcolumntype{C}{>{\centering\arraybackslash}X}
\newcommand{\tablefont}{\fontsize{7.0}{9.25}\selectfont}
\newcommand{\tablenote}[1]{\par\vspace{2pt}{\fontsize{6.6}{8.35}\selectfont\raggedright #1\par}\vspace{2pt}\hrule height .4pt\relax}
\renewcommand{\arraystretch}{1.13}
\setlist[itemize]{leftmargin=*,itemsep=2pt,topsep=3pt}
\setlist[enumerate]{leftmargin=*,itemsep=2pt,topsep=3pt}
\newcommand{\conceptfigure}[3]{%
  \IfFileExists{figures/#1.pdf}{\includegraphics[width=\linewidth]{figures/#1.pdf}}{%
  \begingroup\setlength{\fboxsep}{0pt}\color{ink}%
  \fbox{\begin{minipage}[c][#2][c]{\dimexpr\linewidth-2\fboxrule\relax}%
    \centering\small #3\par\vspace{3mm}%
    \footnotesize Reserved composition; no synthetic experimental image.
  \end{minipage}}\endgroup}}

\graphicspath{{figures/compact/}{figures/experimental/}{figures/}}
\definecolor{linkpurple}{HTML}{FF1493}
\hypersetup{colorlinks=true,linkcolor=linkpurple,citecolor=linkpurple,urlcolor=linkpurple,
  pdftitle={D-JEPA: A Decision-Aligned Latent World Model},pdfauthor={}}
\newtheorem{proposition}{Proposition}
\newtheorem{lemma}{Lemma}
\newtheorem{corollary}{Corollary}
\title{\titlewithicon{D-JEPA: A Decision-Aligned Latent}{World Model}}
\ifdefined\djepaauthorversion
  \iclrfinalcopy
  \author{%
\makebox[\textwidth][c]{%
\begin{tabular}{c}
Shuaijun Liu$^{1}$ \quad Chengyu Wu$^{1}$ \quad Qifu Wen$^{2,3}$ \quad Feiyang You$^{1}$ \\
Chenglong Zhang$^{1}$ \quad Shuyang Hao$^{1}$ \quad Xi Lin$^{3}$ \quad Ningxin Su$^{1,*}$ \\[0.45em]
{\normalfont\small $^{1}$The Hong Kong University of Science and Technology (Guangzhou)} \\
{\normalfont\small $^{2}$Boston University \quad $^{3}$Shanghai Jiao Tong University} \\
{\normalfont\small $^{*}$Corresponding author: \href{mailto:ningxinsu@hkust-gz.edu.cn}{\textcolor{gray}{ningxinsu@hkust-gz.edu.cn}}} \\
{\normalfont\small Project Website: \href{https://nebulis-lab.com/D-JEPA}{https://nebulis-lab.com/D-JEPA}}
\end{tabular}%
}}

  \hypersetup{pdfauthor={Shuaijun Liu, Chengyu Wu, Qifu Wen, Feiyang You, Chenglong Zhang, Shuyang Hao, Xi Lin, Ningxin Su}}
\fi
\makeatletter
\ifdefined\djepaauthorversion
\patchcmd{\@maketitle}{Published as a conference paper at ICLR 2027}{Preprint}%
{}{\PackageError{D-JEPA}{Author-version header could not be patched}{Check the title definition in the ICLR style.}}
\patchcmd{\@maketitle}{\@title\par}{\@title\par\vskip 10pt}%
{}{\PackageError{D-JEPA}{Title-to-author spacing could not be patched}{Check the title definition in the ICLR style.}}
\else
\patchcmd{\@maketitle}{Anonymous authors}{%
  Anonymous authors\hspace{1.2em}{\normalfont\normalsize\bfseries\textbar}\hspace{1.2em}{\normalfont\normalsize\bfseries Anonymous Website:
  \href{https://anonymous-for-sub.github.io/D-JEPA}{\textmd{https://anonymous-for-sub.github.io/D-JEPA}}}%
}{}{\PackageError{D-JEPA}{Anonymous author line could not be patched}{Check the title definition in the ICLR style.}}
\patchcmd{\@maketitle}{Paper under double-blind review}{%
  \underline{\normalfont\itshape Paper under double-blind review}%
}{}{\PackageError{D-JEPA}{Review notice could not be patched}{Check the title definition in the ICLR style.}}
\fi
\patchcmd{\@maketitle}{\vskip 0.3in minus 0.1in}{\vskip 0.12in}%
{}{\PackageError{D-JEPA}{Title-to-overview spacing could not be patched}{Check the title definition in the ICLR style.}}
\makeatother
\date{}
\begin{document}
\begingroup
\setlength{\tabcolsep}{0pt}
\maketitle
\endgroup
\noindent
\begin{minipage}{\textwidth}
\centering
\includegraphics[width=\linewidth]{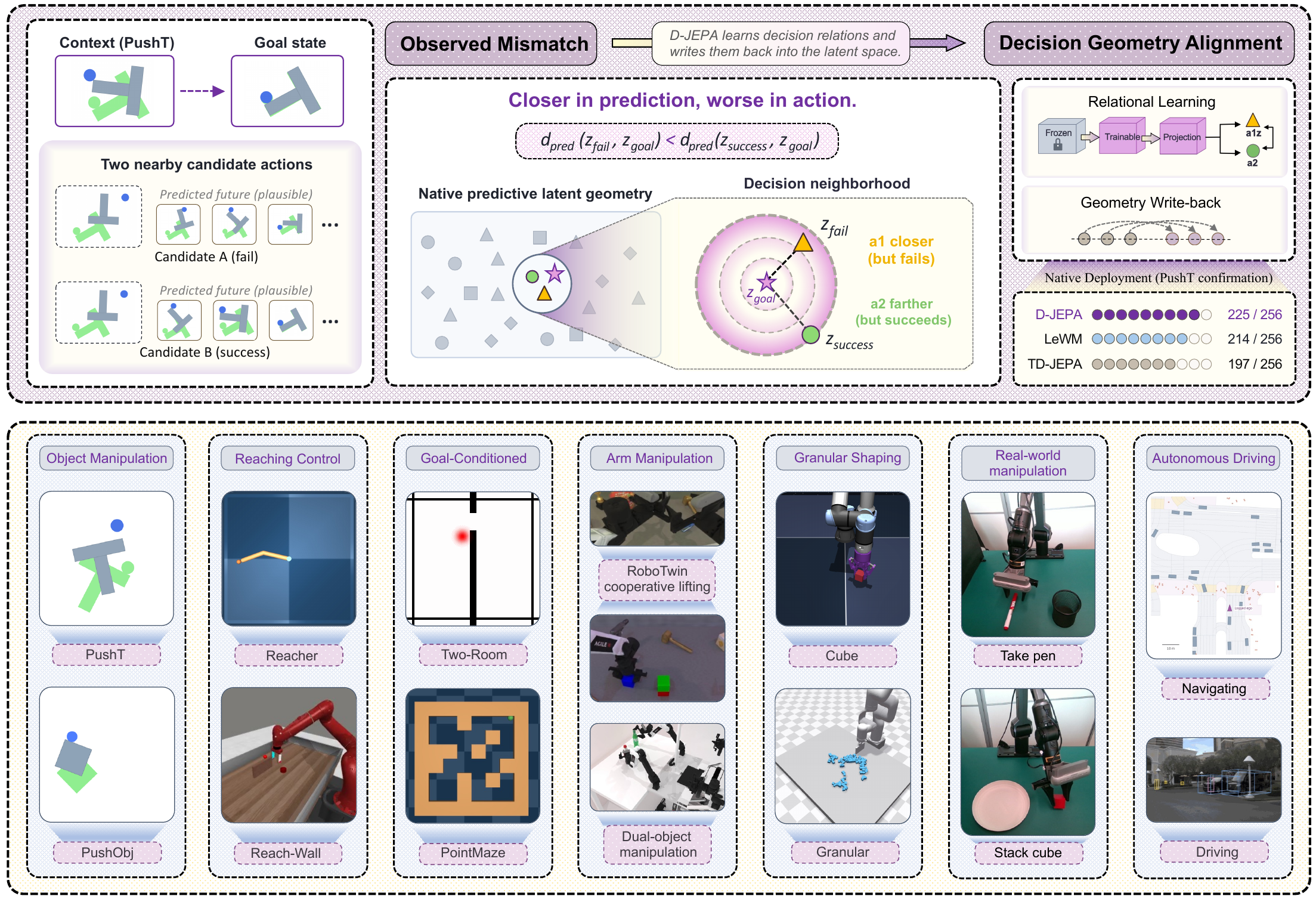}
\captionof{figure}{\textbf{From predictive proximity to decision-aligned control.} A candidate closer to the goal in predictive latent space can fail while a farther alternative succeeds. D-JEPA learns relations among futures and realizes decision structure in future representations. The PushT confirmation summarizes action-selection performance; the lower strip situates the framework across manipulation, reaching, goal-conditioned control and driving.}
\label{fig:overview}
\vspace{3mm}
\end{minipage}
\par
\begin{abstract}
Latent world models predict the consequences of actions, but accurate prediction does not guarantee that latent distance reflects which candidate will execute successfully. We identify a decision-local prediction gap: among the few futures competing for execution, a candidate predicted closer to the goal can produce a worse realized outcome than an available alternative. We introduce D-JEPA, a decision-aligned latent world model that learns decision-relevant relations among candidate futures from executed outcomes. A bounded, permutation-equivariant operator jointly reasons over goal-relative predictive features and ordinal evidence, refining pretrained predictive geometry where action choices are most consequential. Restricted predictor adaptation and a shared ordinal interface extend this alignment across complementary predictive geometries. D-JEPA further realizes the learned decision structure in JEPA-compatible future representations, enabling deployment through native latent-distance planning. Evaluations across latent control, manipulation, pretrained action-producing models, physical robots and autonomous driving demonstrate improved action selection, including 87.89\% success on PushT, a 15.04-point average gain on RoboTwin, and a 17-point gain on physical robot tasks. These results establish decision-relevant relational structure as a direct bridge between predictive world modeling and effective control.
\end{abstract}

\section{Introduction}
World models turn observations and hypothetical actions into predictions of what could happen next. Joint-embedding predictive architectures predict in a learned representation space rather than reconstructing pixels~\citep{assran2023ijepa,assran2025vjepa2,maes2026lewm}. A goal-conditioned planner compares these predicted futures with a goal representation and executes the preferred sequence. Predictive geometry therefore determines which action reaches the environment.

A smaller predicted goal distance can favor an action that fails over an alternative that succeeds. This conflict is consequential among the few candidate futures competing for execution. On matched PushT decisions, within-start Spearman correlation between predicted and realized latent costs falls from 0.90 to 0.11 for LeWM and 0.80 to 0.13 for TD-JEPA as the full set narrows to the preferred four (Figure~\ref{fig:mismatch}). This \emph{decision-local prediction gap} complements work on plan-cost fidelity and planning-aware latent metrics~\citep{you2026control,bai2026temporal}. We turn the mismatch into a concrete learning problem: executed candidate outcomes supervise relations among the futures that compete to determine the next action.

We introduce \method{} (Figure~\ref{fig:overview}), which learns a bounded relational correction over the complete candidate set. Its central operator jointly processes goal-relative predictive features and ordinal evidence, sharpening distinctions near the decision boundary. Restricted predictor adaptation supplies a complementary proposal to the relational operator. A common ordinal interface integrates evidence across heterogeneous predictive geometries.

D-JEPA further realizes the learned decision structure in JEPA-compatible future representations. Bounded temporal transport learns same-action changes along a predicted trajectory; exact ordinal realization embeds this structure in terminal goal-relative geometry. The latter makes aligned action choices directly available through the native latent-distance planning interface. We evaluate the framework across latent control, robotic manipulation, pretrained action-producing models, physical robots and autonomous driving, with additional geometric and visual shifts. Matched ablations isolate the core relation operator and its complementary predictive evidence.

Our contributions are: (1) \textbf{Decision-local diagnosis.} We identify and quantify a mismatch between predicted goal distance and executed outcomes near action selection, and formulate candidate-outcome supervision for this regime. (2) \textbf{Relational decision alignment.} We introduce a bounded, permutation-equivariant operator that learns decision-relevant relations among pretrained candidate futures, complemented by restricted predictor adaptation and an ordinal interface across predictive geometries. (3) \textbf{Decision-aligned future representations.} We realize learned decision structure in JEPA-compatible future geometry, recover action choices through native latent distance, and validate decision alignment across simulated control, pretrained action models, driving and physical robots.

\section{Related Work}
\label{sec:related}
Latent dynamics support both planning and behavior learning. PlaNet plans through learned latent transitions, Dreamer learns behavior through latent imagination, and TD-MPC couples task-oriented dynamics with value-based control~\citep{hafner2019planet,hafner2025mastering,hansen2024tdmpc2}. Joint-embedding methods learn representations by predicting target embeddings~\citep{assran2023ijepa}. V-JEPA~2 extends this direction to action-conditioned prediction and planning~\citep{assran2025vjepa2}; DINO-WM uses pretrained visual features for world modeling~\citep{zhou2025dinowm}; and LeWM learns a compact predictive representation end to end~\citep{maes2026lewm}. These models establish increasingly capable predictive futures. D-JEPA learns decision-relevant relations among the predicted futures that directly compete for execution.

Planning-aware representation learning makes predictive geometry explicitly consequential for control. TD-JEPA mines directed temporal progress and studies task-dependent deployment of temporal and Euclidean costs~\citep{bai2026temporal}, while A Control Theory of Predictability connects plan quality to cost fidelity on planner-reachable candidates~\citep{you2026control}. D-JEPA turns this interface into a supervised candidate-set learning problem: executed consequences teach a complete-set operator which relative distinctions matter among alternatives sharing the same observed context and goal, and representation lifting exposes the learned order through a latent readout.

World-model design studies provide complementary sources of predictive structure. JEPA-WM systematically examines architectures, objectives and planning choices~\citep{terver2026jepawm}, while Fast LeWorldModel changes the prediction operator to action-prefix prediction~\citep{gao2026fast}. D-JEPA provides a shared decision-alignment interface over such sources: dense descriptors remain inside each predictive space, while ordinal coordinates communicate complementary evidence across heterogeneous geometries. This construction improves action selection while preserving the semantics of each source geometry.

Visuomotor methods such as Diffusion Policy learn action distributions directly from demonstrations~\citep{chi2023diffusion}. D-JEPA operates at the complementary decision stage, learning which available predictive future should guide execution. The same relational principle therefore applies to latent planners, pretrained VLA action chunks, Drive-JEPA trajectories and V-JEPA~2-AC candidates. Section~\ref{sec:problem} next formalizes and measures the decision-local prediction gap that motivates this design.

\section{The Decision-Local Prediction Gap}
\label{sec:problem}
We begin from the interface exposed by the preceding predictive models: a set of action-conditioned futures whose relative geometry determines which action is executed. Let $x$ denote the observed context, $g$ a goal observation, and $\mathcal A=\{a_i\}_{i=1}^{K}$ a set of candidate action sequences. Predictive model $m$ produces future latents $\hat z^m_{i,1:H}=F_m(x,a_i)$ and goal embedding $z_g^m=E_m(g)$. Its native planning cost is $c_i^m=C_m(\hat z^m_{i,1:H},z_g^m)$; a common choice is terminal mean-squared latent distance. The selected action is $a_{\argmin_i c_i^m}$. Executing $a_i$ from the same restored state supplies an outcome, a task cost, and a success label $y_i\in\{0,1\}$.

A recorded PushT pair exposes a direct preference conflict: LeWM assigns a smaller predicted goal RMS distance to a failing candidate than to a successful alternative (Figure~\ref{fig:mismatch}a). Across the 96-start audit, each model prefers a failing candidate on eight starts despite an available successful alternative (Figure~\ref{fig:mismatch}b). We then measure success--failure pair inversions within each model's preferred shortlist. Among mixed-outcome top-four shortlists, mean inversion rates reach 49.0\% for LeWM and 38.2\% for TD-JEPA (Figure~\ref{fig:mismatch}c).

The complementary cost diagnostic compares predicted costs with costs from simulator-realized observations under the same backbone's latent goal criterion. Within-start Spearman correlation drops below 0.13 for both models among their four lowest-cost candidates (Figure~\ref{fig:mismatch}d). Predictive geometry remains informative across the full set, yet weakly distinguishes the actions closest to execution. Appendix~\ref{app:distanceconflict} defines the outcome-based diagnostics; Table~\ref{tab:diagnostic} gives all correlation coefficients.

\begin{figure}[t]
\centering
\begin{subfigure}[t]{0.24\linewidth}\includegraphics[width=\linewidth]{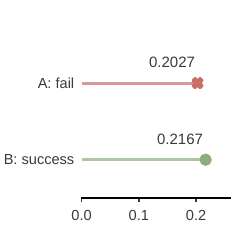}\caption{Recorded distances.}\end{subfigure}\hfill
\begin{subfigure}[t]{0.24\linewidth}\includegraphics[width=\linewidth]{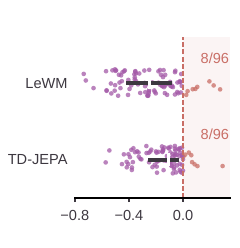}\caption{Distance gaps.}\end{subfigure}\hfill
\begin{subfigure}[t]{0.24\linewidth}\includegraphics[width=\linewidth]{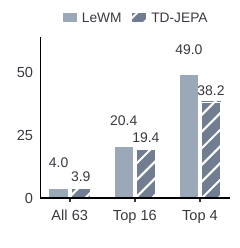}\caption{Pair inversions (\%).}\end{subfigure}\hfill
\begin{subfigure}[t]{0.24\linewidth}\includegraphics[width=\linewidth]{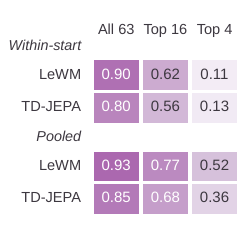}\caption{Rank correlation.}\end{subfigure}
\caption{\textbf{Predicted distance preferences can conflict with executed outcomes.} (a) Predicted goal RMS distance favors failing A over successful B in a recorded LeWM pair. (b) Per-start normalized distance gaps between the closest successful and failing candidates; positive values favor failure. Points show all 96 audit starts per model; bars mark medians and interquartile intervals. (c) Mean success--failure pair inversion rates within mixed-outcome shortlists; top-four means use 34 LeWM and 29 TD-JEPA starts. (d) Predicted--realized latent-cost correlations on the same audit. The illustrative pair in (a) belongs to a separate confirmation start. Definitions and all denominators appear in Appendix~\ref{app:distanceconflict}.}
\label{fig:mismatch}
\end{figure}

This diagnosis defines a direct learning target. Training-time counterfactual execution identifies successful and unsuccessful alternatives within each candidate set, and fitting data teach their decision relations. At deployment, the model receives context-derived predictive evidence and selects an action whose executed outcome is evaluated independently. Fixed candidate pools hold action availability constant, while the relational formulation applies to a general set of size $K$.

Global predictive agreement alone cannot guarantee correct action selection (Proposition~\ref{prop:globalgap}). We retain complete-set predictive context while concentrating supervision on the decision-relevant subset and bounding changes to the underlying representation or decision score.

\section{D-JEPA}
\label{sec:method}
D-JEPA aligns pretrained predictive geometry with decision outcomes through relational candidate reasoning (Figure~\ref{fig:method}), combines complementary predictive evidence (Figure~\ref{fig:composition}), and realizes the aligned structure in future latents (Figure~\ref{fig:lifting}). Appendix~\ref{app:notation} summarizes the notation.

\begin{figure}[t]
\centering
\includegraphics[width=\linewidth]{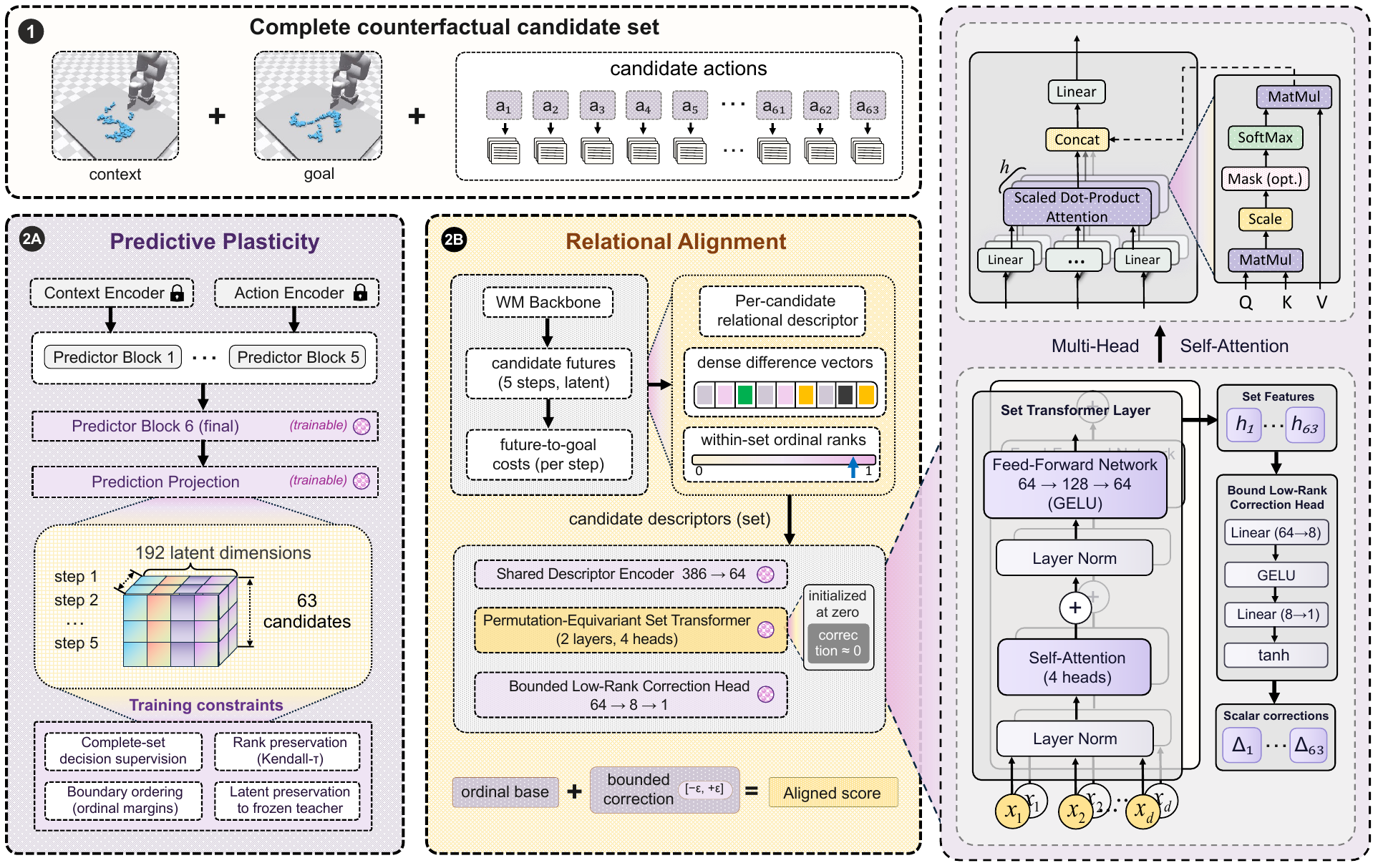}
\caption{\textbf{Learning decision-relevant relations over candidate futures.} Candidates share a context and goal. Predictive plasticity adapts the predictor tail and projection; relational alignment combines goal-relative descriptors and ordinal evidence through a permutation-equivariant set operator and a bounded correction. Both paths retain candidate identities. The diagram shows the 63-candidate configuration; Figure~\ref{fig:composition} continues with multi-geometry evidence and calibrated composition.}
\label{fig:method}
\end{figure}
\begin{figure}[t]
\noindent
\begin{minipage}[t]{0.49\linewidth}\vspace{0pt}
\centering
\includegraphics[width=\linewidth]{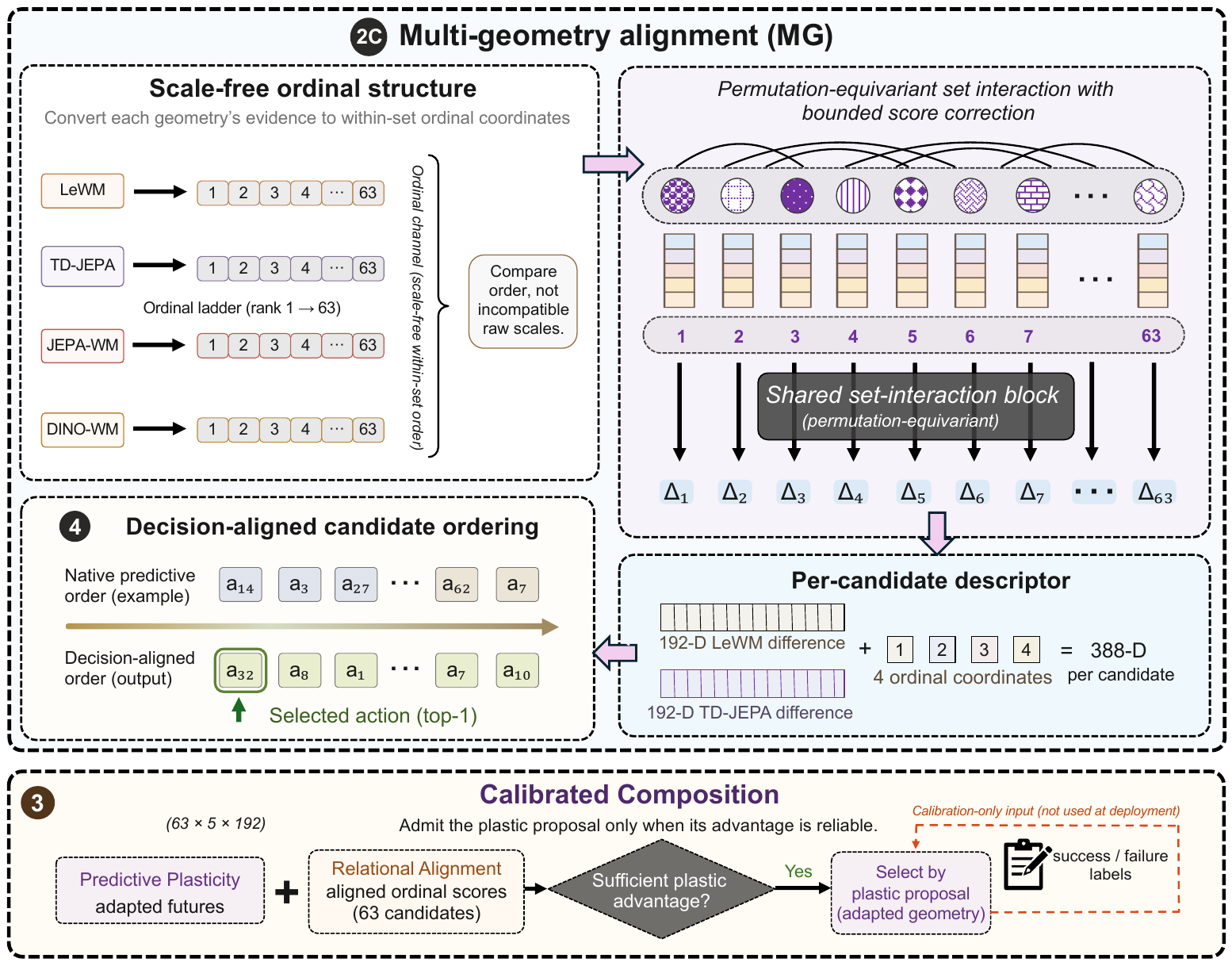}
\caption{\textbf{Combining geometries and proposals.} Scale-free ordinal evidence feeds the relational operator; calibrated composition admits a reliable plastic proposal over the relational default. Candidate identities remain unchanged.}
\label{fig:composition}
\end{minipage}\hfill
\begin{minipage}[t]{0.49\linewidth}\vspace{0pt}
\centering
\includegraphics[width=\linewidth]{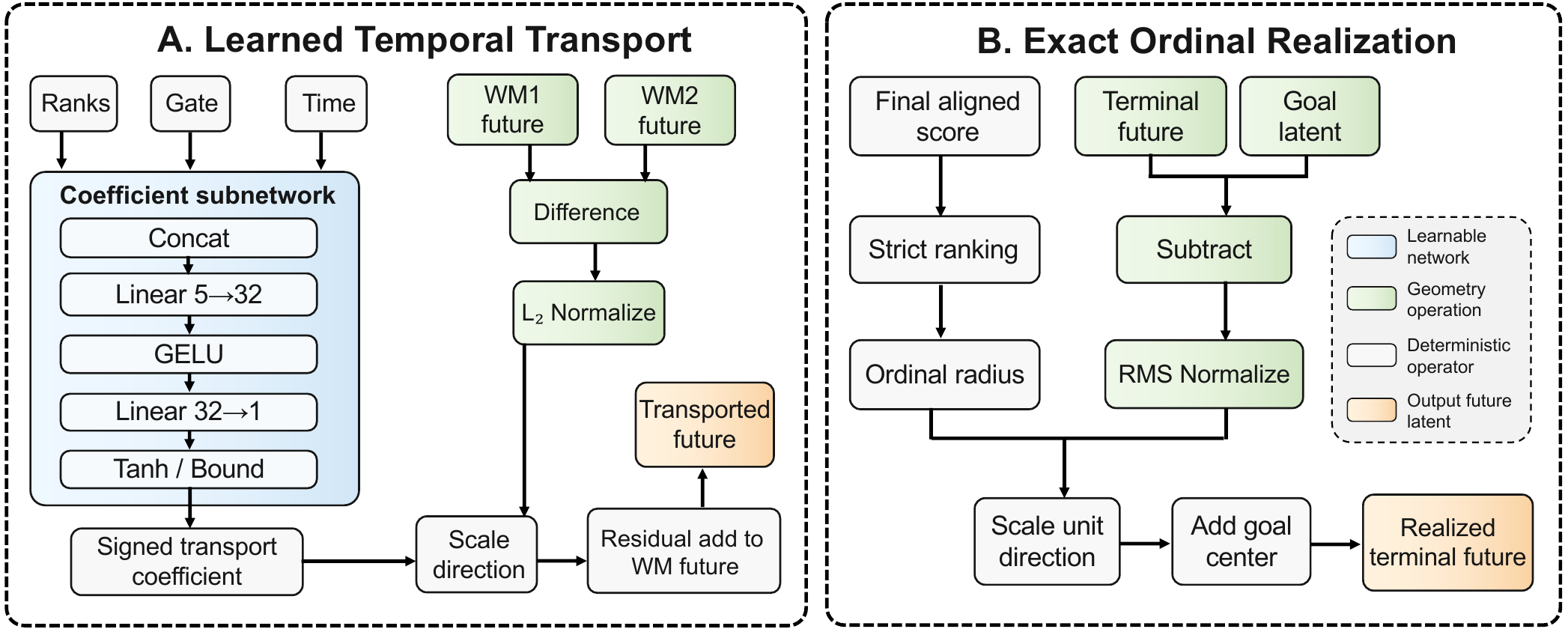}
\caption{\textbf{Realizing decision structure through native future geometry.} Temporal transport concatenates source ranks, inferred relational rank, gate state and time. A bounded coefficient network scales the normalized difference between same-action model futures before residual addition. The ordinal branch maps the final aligned score to a strict rank, sets the terminal root-mean-square radius, and restores the goal centre. Earlier steps remain unchanged under terminal realization. Colours distinguish learnable and deterministic operations. Realized futures recover the aligned action through native goal distance (Figure~\ref{fig:liftingdetail}).}
\label{fig:lifting}
\end{minipage}
\end{figure}

\subsection{Predictive evidence and bounded relational alignment}
\label{sec:relational}
For model $m$, candidate $i$ supplies a normalized goal-relative descriptor $d_i^m=\operatorname{LN}(\hat z^m_{i,H}-z_g^m)$ and an ordinal coordinate $r_i^m=(\rank_{\mathcal A}(c_i^m)-1)/(K-1)$. Ranks deterministically read pretrained native costs; execution labels supervise alignment only. Descriptors retain within-model directions; ranks are scale-free. The dual-model token is $v_i=[d_i^L;d_i^T;r_i^L;r_i^T]$, with base score $b_i=\alpha r_i^L+(1-\alpha)r_i^T$. A shared encoder and two permutation-equivariant Transformer layers process the complete set. Output $h_i$ receives a zero-initialized, rank-eight correction:
\begin{equation}
 \delta_i=\epsilon\tanh\!\left(W_{\rm up}\tanh(W_{\rm down}h_i)\right),
 \qquad s_i=b_i+\delta_i,\qquad \epsilon=0.2.
 \label{eq:bounded}
\end{equation}
The update preserves preferences across base-score gaps above $2\epsilon$ and selects within $2\epsilon$ of the base minimum (Proposition~\ref{prop:boundedalignment}). With $p_i=\exp(-s_i/T)/\sum_j\exp(-s_j/T)$, training maximizes decision mass on successful alternatives and sharpens local success/failure ordering:
\begin{equation}
 \mathcal L_R=-\log\sum_{i:y_i=1}p_i
 +\lambda_{\rm local}\mathcal L_{\rm local}
 +\lambda_{\rm trust}K^{-1}\sum_i\delta_i^2.
 \label{eq:relational}
\end{equation}
Together, the objective rewards successful candidates globally, sharpens success/failure ordering locally near the current decision boundary, and regularizes unnecessary global reordering. Local pairs use the current low-score subset; the operator receives all candidates. A calibrated margin gate selects between the relational and base winners. Full objectives and tie handling appear in Appendix~\ref{app:methoddetails}.

\begin{table}[t]
\centering
\caption{\textbf{Decision-aligned action selection across predictive world models.} Success (\%) on the PushT confirmation cohort ($n=256$), Reacher ($n=128$) and Granular ($n=64$), sharing candidate pools within tasks. Granular reports Chamfer distance (CD); the last column is an earlier PushT diagnostic ($n=32$).}
\label{tab:main}\label{tab:comprehensive}
\fontsize{6.7}{8.8}\selectfont
\setlength{\tabcolsep}{2.4pt}
\renewcommand{\arraystretch}{1.17}
\begin{tabularx}{\linewidth}{>{\hsize=3.6\hsize\linewidth=\hsize}L*{6}{>{\hsize=.5666667\hsize\linewidth=\hsize\raggedleft\arraybackslash}X}}
\toprule
\rowcolor{headerbg}\textbf{Method and publication} & \makecell{\textbf{PushT}\\\textbf{$\uparrow$}} & \makecell{\textbf{Reacher}\\\textbf{$\uparrow$}} & \makecell{\textbf{Granular}\\\textbf{CD $\downarrow$}} & \makecell{\textbf{Main}\\\textbf{$\uparrow$}} & \makecell{\textbf{Strict}\\\textbf{$\uparrow$}} & \makecell{\textbf{Earlier}\\\textbf{PushT $\uparrow$}}\\
\midrule
\rowcolor{sectiongray}\multicolumn{7}{l}{\textit{Pretrained predictive world models}}\\
JEPA-WM \pubref{Terver et al.}{TMLR 2026}{terver2026jepawm} & 85.16 & 84.38 & 0.370 & 34.38 & 12.50 & 81.25\\
DINO-WM \pubref{Zhou et al.}{ICML 2025}{zhou2025dinowm} & 82.03 & 80.47 & \textbf{0.341} & 35.94 & 7.81 & 75.00\\
\rowcolor{referencebg}LeWM \yearref{Maes et al.}{2026}{maes2026lewm} & 83.59 & 86.72 & 0.350 & 37.50 & 14.06 & 81.25\\
TD-JEPA \yearref{Bai et al.}{2026}{bai2026temporal} & 76.95 & 68.75 & 0.390 & 28.13 & 9.38 & 84.38\\
\midrule
\rowcolor{sectiongray}\multicolumn{7}{l}{\textit{Matched adaptations of published objectives}}\\
PEGrad \pubref{Peri et al.}{CoRL 2025}{peri2025pegrad} & 82.03 & 78.13 & 0.380 & 31.25 & 10.94 & 65.63\\
Var-JEPA \pubref{G\"ogl et al.}{ICML 2026}{gogl2026varjepa} & 80.08 & 75.00 & 0.400 & 29.69 & 9.38 & 62.50\\
Conflict-safe variational adaptation \pubref{G\"ogl et al.}{ICML 2026}{gogl2026varjepa} & 81.25 & 77.34 & 0.390 & 31.25 & 10.94 & 65.63\\
Expected-value planning \pubref{Enwerem et al.}{IROS 2026}{enwerem2026vnb} & 78.91 & 74.22 & 0.410 & 28.13 & 9.38 & 62.50\\
CVaR planning \pubref{Enwerem et al.}{IROS 2026}{enwerem2026vnb}; \pubref{Ni et al.}{ICML 2024}{ni2024cvar} & 79.69 & 75.78 & 0.400 & 29.69 & 9.38 & 62.50\\
\midrule
Four-geometry fusion & 85.94 & 82.03 & 0.360 & 37.50 & 15.63 & 78.13\\
\rowcolor{oursbg}\textbf{D-JEPA} & \textbf{87.89} & \textbf{93.75} & 0.358 & \textbf{39.06} & \textbf{18.75} & \textbf{93.75}\\
\bottomrule
\end{tabularx}
\tablenote{Publications identify source methods; values are from matched evaluations. Granular thresholds are 0.18 (main) and 0.12 (strict). PushT uses calibrated composition; Reacher and Granular use relational selection. The earlier PushT population is evaluated separately.}
\end{table}

\begin{figure}[t]
\centering
\begin{subfigure}[t]{0.48\linewidth}\vspace{0pt}
\includegraphics[width=\linewidth]{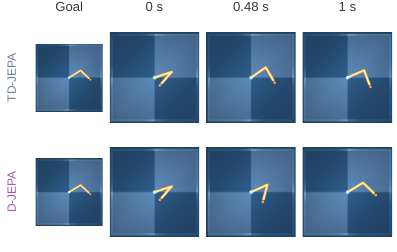}
\caption{Articulated control.}
\end{subfigure}\hfill
\begin{subfigure}[t]{0.48\linewidth}\vspace{0pt}
\includegraphics[width=\linewidth]{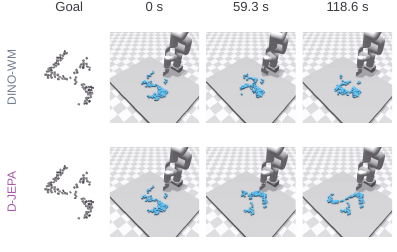}
\caption{Granular redistribution.}
\end{subfigure}
\caption{\textbf{Decision alignment across distinct physical dynamics.} Matched initial states, goals and execution horizons reveal the physical consequences of different selections. D-JEPA reaches the Reacher target and redistributes the Granular particles toward the prescribed goal. Granular panels share a timeline, with the completed rollout's terminal state held through the remaining timestamps.}
\label{fig:additionaldynamics}\label{fig:reacherqual}\label{fig:granularqual}
\end{figure}

\begin{figure}[t]
\begin{minipage}[t]{0.49\linewidth}\vspace{0pt}
\captionof{table}{\textbf{Transfer across robot embodiments.} Success (\%) with native selection and D-JEPA.}
\label{tab:realrobot}\label{tab:robotwin}
{\fontsize{6.8}{9.0}\selectfont\setlength{\tabcolsep}{2.2pt}\renewcommand{\arraystretch}{1.10}
\begin{tabularx}{\linewidth}{Lrrrr}
\toprule
\rowcolor{headerbg}\textbf{Task} & \textbf{$n$} & \textbf{Native} & \textbf{D-JEPA} & \textbf{$\Delta$ pp}\\
\midrule
\rowcolor{sectiongray}\multicolumn{5}{l}{\textit{RoboTwin}\enspace\paperbadge{VLA}}\\
Grab roller & 128 & 71.09 & \textbf{83.59} & +12.50\\
Bread to skillet & 128 & 64.06 & \textbf{78.91} & +14.84\\
Object to cabinet & 128 & 59.38 & \textbf{75.00} & +15.63\\
Block handover & 128 & 52.34 & \textbf{69.53} & +17.19\\
\rowcolor{sectiongray}\textit{Overall} & 512 & 61.72 & \textbf{76.76} & +15.04\\
\midrule
\rowcolor{sectiongray}\multicolumn{5}{l}{\textit{Physical PiPER}\enspace\paperbadge{V-JEPA 2-AC}}\\
Real PushT & 50 & 56.0 & \textbf{74.0} & +18.0\\
Two-cube stack & 50 & 72.0 & \textbf{88.0} & +16.0\\
\rowcolor{sectiongray}\textit{Overall} & 100 & 64.0 & \textbf{81.0} & +17.0\\
\bottomrule
\end{tabularx}}
\tablenote{Physical trials use paired starts and a shared V-JEPA~2-AC planner. Gains/losses are 12/3 on PushT and 10/2 on stacking.}

\end{minipage}\hfill
\begin{minipage}[t]{0.48\linewidth}\vspace{0pt}
\captionof{figure}{\textbf{Aligned bimanual grasping.} D-JEPA completes the grasp from an initial state where native selection fails.}
\label{fig:robotwincase}\label{fig:foundationcases}
\includegraphics[width=\linewidth]{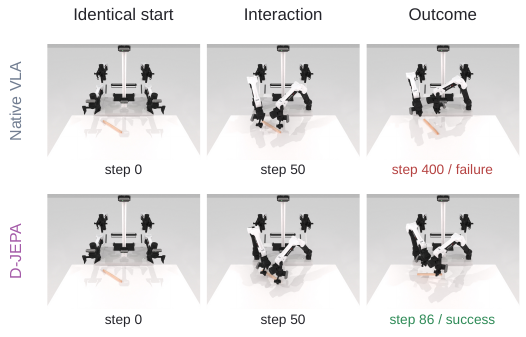}
\end{minipage}
\end{figure}

\begin{figure}[t]
\begin{minipage}[t]{0.49\linewidth}\vspace{0pt}
\captionof{table}{\textbf{Driving trajectory selection.} PDMS ($\uparrow$); 32 shared trajectories per scene, evaluated under the same traffic context.}
\label{tab:driving}
{\fontsize{6.8}{9.0}\selectfont\setlength{\tabcolsep}{2pt}\renewcommand{\arraystretch}{1.85}
\begin{tabularx}{\linewidth}{p{19pt}*{4}{>{\raggedleft\arraybackslash}X}}
\toprule
\rowcolor{headerbg}\textbf{Scene} & \textbf{Drive-JEPA} & \textbf{D-JEPA} & \textbf{$\Delta$} & \textbf{Oracle}\\
\midrule
S1 & 0.00 & \textbf{80.00} & +80.00 & 100.00\\
S2 & 58.33 & \textbf{95.83} & +37.50 & 100.00\\
S3 & 58.33 & \textbf{95.83} & +37.50 & 100.00\\
S4 & 58.33 & \textbf{99.17} & +40.84 & 100.00\\
S5 & 58.33 & \textbf{99.17} & +40.84 & 100.00\\
S6 & 83.33 & \textbf{99.67} & +16.34 & 100.00\\
S7 & 84.86 & \textbf{97.73} & +12.87 & 100.00\\
\rowcolor{sectiongray}\textit{Mean} & 57.36 & \textbf{95.34} & +37.98 & 100.00\\
\bottomrule
\end{tabularx}}
\tablenote{Drive-JEPA uses native ranking; Oracle is the candidate-outcome ceiling. Each row is a distinct difficult scene. The seven scenes span three source logs and receive equal weight in the mean (Appendix~\ref{app:metrics}).}

\end{minipage}\hfill
\begin{minipage}[t]{0.48\linewidth}\vspace{0pt}
\captionof{figure}{\textbf{Aligned driving trajectories.} D-JEPA maintains greater clearance. Shared logged camera observations appear above simulated candidate executions at matched scales and times.}
\label{fig:drivingcase}
\includegraphics[width=\linewidth]{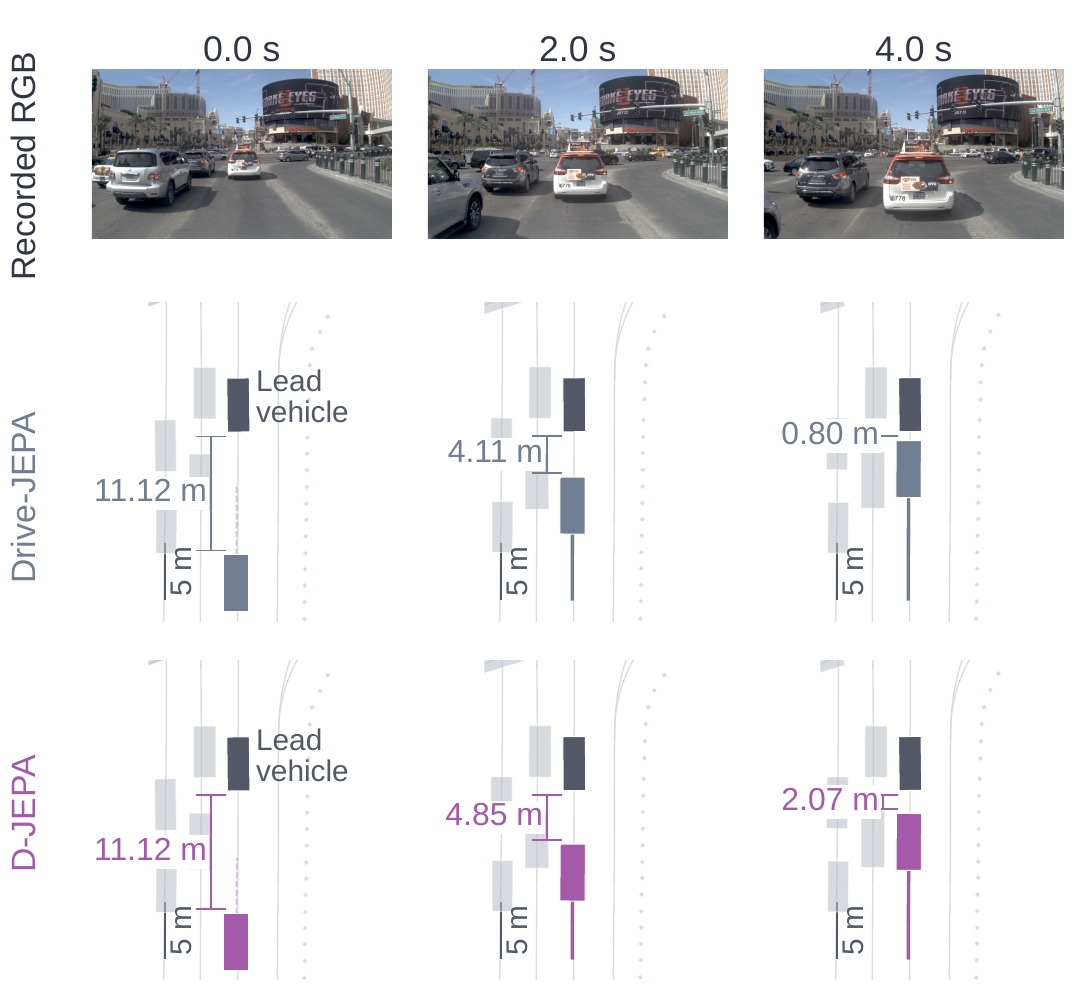}
\end{minipage}
\end{figure}

\begin{table}[tbp]
\centering
\caption{\textbf{Decision alignment under shape and visual changes.} Both experiments use a learned task-local relational model and matched predictive features. Entries are success percentages; the reference controls are native predictive cost and calibrated visual/proprioceptive rank fusion.}
\label{tab:generalization}
\tablefont
\begin{tabularx}{\linewidth}{p{34mm}*{5}{C}}
\toprule
\rowcolor{sectiongray}\multicolumn{6}{l}{\textit{PushObj: three unseen geometries}\enspace\paperbadge{Held-out shape}}\\
\rowcolor{headerbg}\textbf{Method} & \textbf{I} & \textbf{Small T} & \textbf{Square} & \textbf{Mean} & \makecell{\textbf{$\Delta$ vs.}\\\textbf{fusion}}\\
\midrule
Native predictive selection & 22.00 & 39.00 & 24.00 & 28.33 & $-16.34$\\
\rowcolor{referencebg}Calibrated rank fusion & 37.00 & 61.00 & 36.00 & 44.67 & 0.00\\
\rowcolor{oursbg}\textbf{D-JEPA} & \bestcell{52.00} & \bestcell{66.00} & \bestcell{43.00} & \bestcell{53.67} & \bestcell{$+9.00$}\\
\bottomrule
\end{tabularx}
\vspace{4pt}
\fontsize{6.6}{8.55}\selectfont
\begin{tabularx}{\textwidth}{p{31mm}*{9}{C}}
\toprule
\rowcolor{sectiongray}\multicolumn{10}{l}{\textit{PushT: fresh starts under seven appearance conditions}\enspace\paperbadge[badgeyellow]{Visual shift}}\\
\rowcolor{headerbg}\textbf{Method} & \textbf{Clean} & \textbf{Blur} & \makecell{\textbf{Salt}\\\textbf{noise}} & \textbf{Dark} & \makecell{\textbf{Object}\\\textbf{colour}} & \makecell{\textbf{Goal}\\\textbf{colour}} & \makecell{\textbf{Pusher}\\\textbf{colour}} & \textbf{Mean} & \makecell{\textbf{$\Delta$ vs.}\\\textbf{fusion}}\\
\midrule
Native predictive selection & 64.00 & 66.00 & 68.00 & 72.00 & 24.00 & 48.00 & 54.00 & 56.57 & $-6.29$\\
\rowcolor{referencebg}Calibrated rank fusion & 70.00 & 76.00 & 72.00 & 74.00 & 30.00 & 56.00 & 62.00 & 62.86 & 0.00\\
\rowcolor{oursbg}\textbf{D-JEPA} & \bestcell{78.00} & \bestcell{82.00} & \bestcell{82.00} & \bestcell{82.00} & \bestcell{54.00} & \bestcell{60.00} & \bestcell{74.00} & \bestcell{73.14} & \bestcell{$+10.29$}\\
\bottomrule
\end{tabularx}
\tablenote{Entries are success percentages. Shape evaluation uses 100 held-out starts per geometry; visual evaluation uses 50 fresh identities per condition. Both protocols select from 63 candidates and execute the complete 25-control sequence.}
\end{table}

\begin{figure}[t]
\centering
\begin{subfigure}[t]{0.24\linewidth}\includegraphics[width=\linewidth]{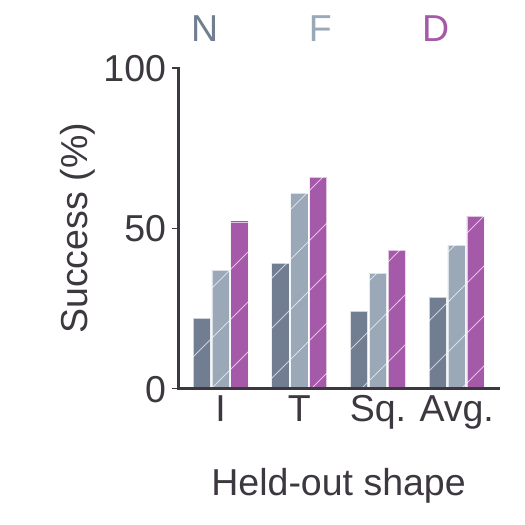}\caption{Shape success.}\end{subfigure}\hfill
\begin{subfigure}[t]{0.24\linewidth}\includegraphics[width=\linewidth]{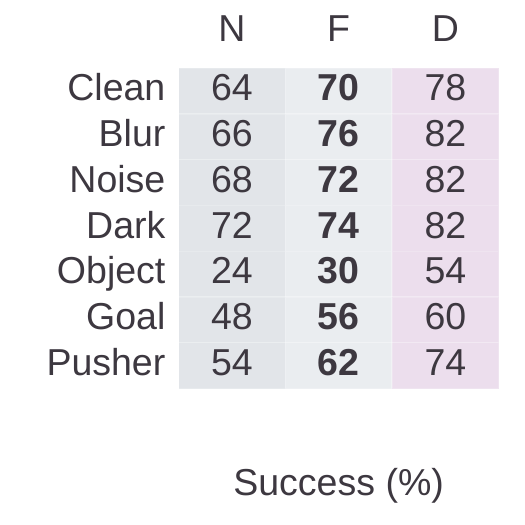}\caption{Visual success.}\end{subfigure}\hfill
\begin{subfigure}[t]{0.24\linewidth}\includegraphics[width=\linewidth]{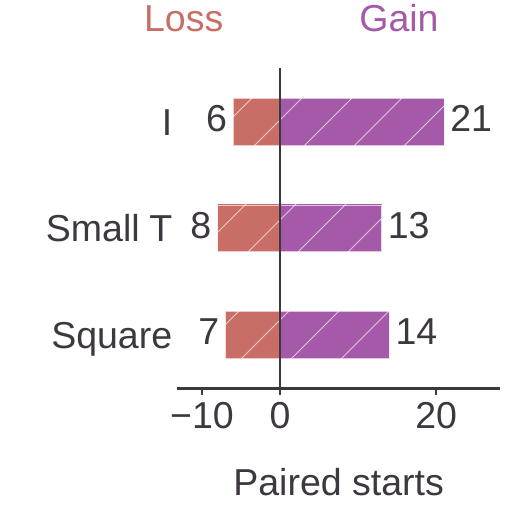}\caption{Shape outcomes.}\end{subfigure}\hfill
\begin{subfigure}[t]{0.24\linewidth}\includegraphics[width=\linewidth]{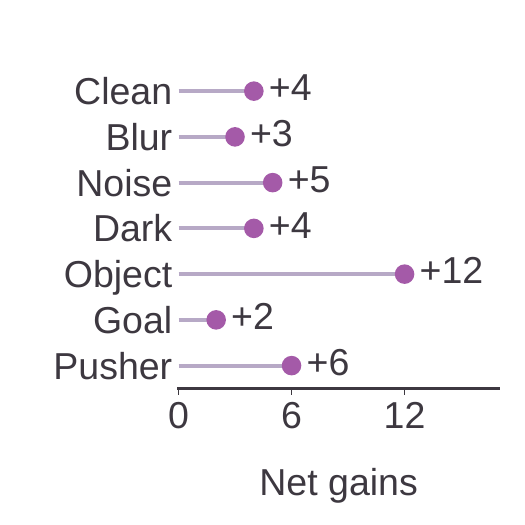}\caption{Visual outcomes.}\end{subfigure}
\caption{\textbf{Decision alignment under geometric and visual change.} (a) Grouped success bars and (b) a condition-by-method success matrix compare native selection (N), calibrated fusion (F) and D-JEPA (D). (c) Diverging bars separate gains from losses; (d) lollipops show their net difference relative to fusion. Each displayed group has positive net gain.}
\label{fig:generalizationvisual}
\end{figure}

\begin{table}[tbp]
\centering
\caption{\textbf{Ablations of decision mechanisms and predictive evidence.} Alignment mechanisms use 256 PushT confirmation starts; predictive-evidence ablations use 128 starts.}
\label{tab:ablation}
\tablefont
\begin{tabularx}{\linewidth}{>{\hsize=1.40\hsize\linewidth=\hsize\raggedright\arraybackslash}X >{\hsize=1.15\hsize\linewidth=\hsize\raggedright\arraybackslash}X >{\hsize=1.15\hsize\linewidth=\hsize\raggedright\arraybackslash}X >{\hsize=.75\hsize\linewidth=\hsize\centering\arraybackslash}X >{\hsize=.55\hsize\linewidth=\hsize\centering\arraybackslash}X}
\toprule
\rowcolor{headerbg}
\textbf{Configuration} & \textbf{Predictive evidence} & \textbf{Decision mechanism} & \textbf{Success (\%)} & \textbf{$\Delta$ (pp)}\\
\midrule
\rowcolor{sectiongray}\multicolumn{5}{l}{\textit{Alignment mechanisms}}\\
TD-JEPA & \paperbadge{Temporal} & \paperbadge[badgeyellow]{Native distance} & 76.95 & 0.00\\
Predictive plasticity & \paperbadge{Adapted temporal} & \paperbadge[badgeyellow]{Native distance} & 79.69 & $+2.74$\\
Relational alignment & \paperbadge{Dual source} & \paperbadge[badgeyellow]{Set relation} & 87.11 & $+10.16$\\
Calibrated composition & \paperbadge{Dual + adapted} & \paperbadge[badgeyellow]{Gated composition} & \bestcell{87.89} & $+10.94$\\
\midrule
\rowcolor{sectiongray}\multicolumn{5}{l}{\textit{Matched predictive-source ablation}}\\
LeWM descriptors & \paperbadge{LeWM} & \paperbadge[badgeyellow]{Set relation} & 84.38 & 0.00\\
Temporal descriptors & \paperbadge{TD-JEPA} & \paperbadge[badgeyellow]{Set relation} & 83.59 & $-0.79$\\
\rowcolor{oursbg}Dual descriptors & \paperbadge{LeWM + TD-JEPA} & \paperbadge[badgeyellow]{Set relation} & \bestcell{93.75} & $+9.37$\\
\midrule
\rowcolor{sectiongray}\multicolumn{5}{l}{\textit{Expanded predictive evidence}}\\
DINO-WM & \paperbadge{DINO-WM} & \paperbadge[badgeyellow]{Native distance} & 95.31 & $-3.13$\\
JEPA-WM & \paperbadge{JEPA-WM} & \paperbadge[badgeyellow]{Native distance} & 98.44 & 0.00\\
Four-geometry fusion & \paperbadge{Four geometries} & \paperbadge[badgeyellow]{Convex fusion} & \bestcell{99.22} & $+0.78$\\
\rowcolor{oursbg}Multi-geometry alignment & \paperbadge{Four geometries} & \paperbadge[badgeyellow]{Set relation} & \bestcell{99.22} & $+0.78$\\
\bottomrule
\end{tabularx}
\tablenote{$\Delta$ uses TD-JEPA, LeWM descriptors and JEPA-WM as the three block references. Confirmation counts for the alignment mechanisms appear in Table~\ref{tab:independentmodules}.}
\end{table}

\begin{figure}[t]
\centering
\begin{subfigure}[t]{0.24\linewidth}\includegraphics[width=\linewidth]{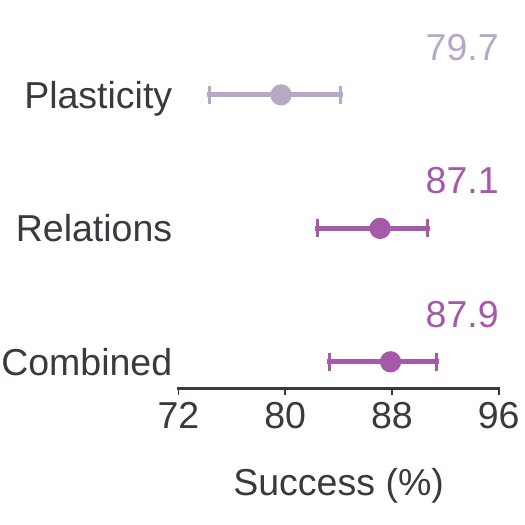}\caption{Components.}\end{subfigure}\hfill
\begin{subfigure}[t]{0.24\linewidth}\includegraphics[width=\linewidth]{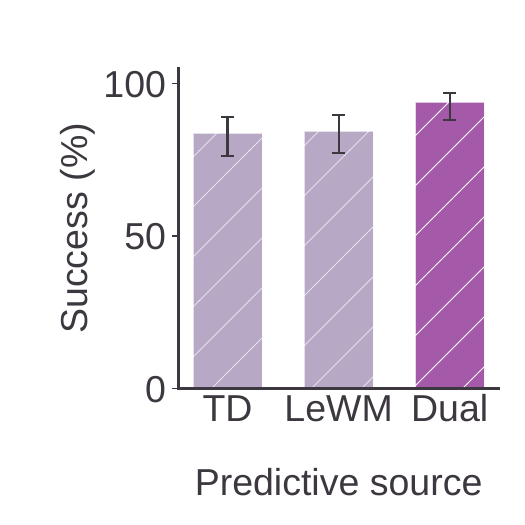}\caption{Predictive sources.}\end{subfigure}\hfill
\begin{subfigure}[t]{0.24\linewidth}\includegraphics[width=\linewidth]{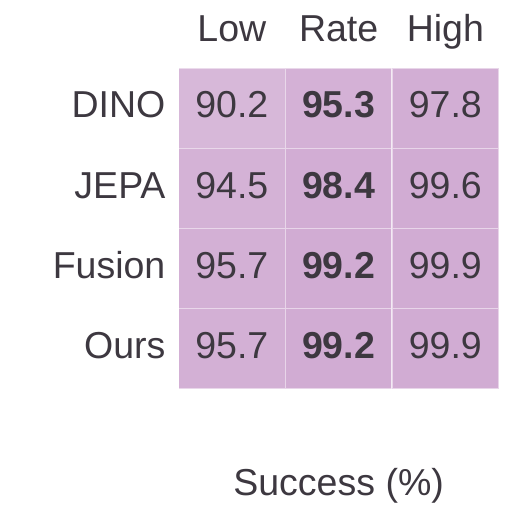}\caption{Geometries.}\end{subfigure}\hfill
\begin{subfigure}[t]{0.24\linewidth}\includegraphics[width=\linewidth]{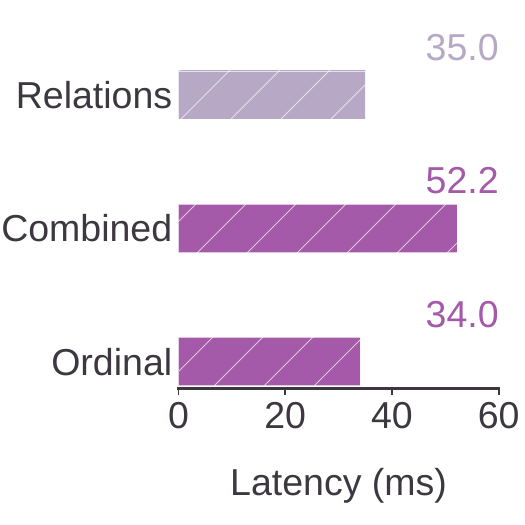}\caption{Inference cost.}\end{subfigure}
\caption{\textbf{Relational alignment benefits from complementary predictive evidence.} (a) Component confirmation; Combined denotes calibrated composition. (b) Predictive-source ablation. (c) Lower limit, success and upper limit across geometries. These panels use Wilson 95\% intervals. (d) Median end-to-end latency. Components, sources and geometries follow the evaluations in Tables~\ref{tab:independentmodules} and~\ref{tab:ablation}.}
\label{fig:componentsvisual}
\end{figure}

\begin{figure}[t]
\begin{minipage}[t]{0.49\linewidth}\vspace{0pt}
\captionof{table}{\textbf{Native-distance deployment.} PushT confirmation evaluation.}
\label{tab:deployment}
{\fontsize{6.9}{9.2}\selectfont\setlength{\tabcolsep}{3pt}
\begin{tabularx}{\linewidth}{Lrr}
\toprule
\rowcolor{headerbg}\textbf{Configuration} & \textbf{Success (\%)} & \textbf{Latency (ms)}\\
\midrule
\rowcolor{referencebg}Relational readout & 87.11 & 35.03\\
Calibrated composition & \textbf{87.89} & 52.16\\
\rowcolor{oursbg}Ordinal realization & 87.11 & \textbf{34.02}\\
\bottomrule
\end{tabularx}}
\tablenote{Interface verification: native goal distance recovers 100\% of relational action choices and candidate ranks after ordinal realization. Median end-to-end inference time is measured per start.}

\smallskip
{\small Figure~\ref{fig:qualitative} shows why decision-aligned predictive geometry matters: both original predictive criteria prefer failing actions, although a successful future is available in the same set under matched execution conditions. D-JEPA identifies that future through its relations to the other candidates. The representation exposes this choice through the predictor's latent-distance interface.}
\end{minipage}\hfill
\begin{minipage}[t]{0.48\linewidth}\vspace{0pt}
\captionof{figure}{\textbf{Different choices, different futures.} Matched PushT executions share a start and goal.}
\label{fig:qualitative}
\centering\includegraphics[width=.97\linewidth]{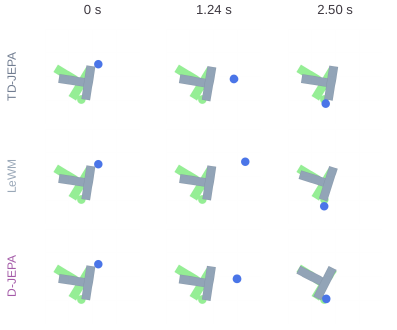}
\end{minipage}
\end{figure}

\subsection{Complementary adaptation and predictive geometries}
\label{sec:plasticity}\label{sec:geometry}
Predictive plasticity adapts only the final TD-JEPA predictor block and projection, allowing future geometry to respond to decision supervision. Success-mass and boundary-ordering objectives combine with ranking and latent-consistency penalties to keep adaptation localized. The resulting native-distance winner supplies a complementary proposal to relational alignment. Calibrated composition accepts this proposal over the relational default when its score advantage exceeds the calibrated threshold.

The ordinal interface extends the relational operator to additional predictive geometries. The four-model instance appends JEPA-WM and DINO-WM ranks to the dual-model descriptors, combining model-specific directions with scale-free evidence of preference. Task-specific predictive evidence also supports alignment over action-producing foundations; Appendix~\ref{app:applications} details their inputs and readouts. Across these interfaces, relations among available futures guide action selection.

\subsection{Representation lifting}
\label{sec:lifting}
Representation lifting realizes learned decision structure directly in future latents through the complementary mechanisms in Figure~\ref{fig:lifting}. For Reacher's same-dimensional models, a small time-conditioned network predicts a signed coefficient $\beta_{i,t}$ bounded by $\rho=0.1$. The same-action temporal update is
\begin{equation}
 \tilde z^T_{i,t}=\hat z^T_{i,t}+\beta_{i,t}
 \frac{\hat z^L_{i,t}-\hat z^T_{i,t}}
 {\max(\|\hat z^L_{i,t}-\hat z^T_{i,t}\|_2,\eta)}.
 \label{eq:transport}
\end{equation}
The network uses source ranks, inferred relational rank, gate state and time, refining five-step trajectories with candidate and time correspondence. Appendix~\ref{app:representation} reports the representation diagnostic.

For PushT, let $\pi_i$ be the strict rank of the final gated score and $u_i$ its original terminal goal-relative direction, normalized to unit root-mean-square norm. Exact ordinal realization defines
\begin{equation}
 \tilde z^T_{i,H}=z_g^T+\frac{\pi_i}{K+1}u_i,\qquad
 \tilde z^T_{i,t}=\hat z^T_{i,t}\quad(t<H).
 \label{eq:ordinal}
\end{equation}
Its native mean-squared goal distance is $(\pi_i/(K+1))^2$, making the learned decision structure recoverable through native goal distance. Selecting the nearest realized future recovers the aligned action through the native planning interface. A self-contained checkpoint includes the predictive and relational paths; numerical rank identity and the zero-displacement convention appear in Appendix~\ref{app:representation}.

\section{Experiments}
\label{sec:experiments}
We evaluate four questions: whether decision alignment improves executed choices across distinct physical dynamics; whether it transfers to action-producing foundations and physical robots; whether learned relations remain effective under geometric and visual changes; and whether aligned decisions can be realized through native future-latent geometry. Core tasks are PushT, DMC-Reacher~\citep{tassa2018control} and Granular goal shaping (Table~\ref{tab:main}), with a separate PushT mechanism population. Transfer covers RoboTwin, physical PiPER manipulation and focused driving scenes. Shift evaluations use unseen PushObj shapes and fresh PushT starts under known appearance conditions.

Methods within each comparison share starts, goals, candidate actions and execution horizons. Candidate executions restore the same initial state. Disjoint fitting, calibration and evaluation identities support learning, model selection and testing, respectively; evaluation uses the calibrated selection rules. Controls include pretrained models, calibrated fusion, matched optimization and uncertainty baselines, and task-specific preservation variants. Physical-robot and driving comparisons retain the reference planner and search budget.
Success is measured over starts or paired physical trials. Appearance conditions retain their repeated-measures structure; driving reports an equal-weight scene mean. Appendix~\ref{app:implementation} details candidate construction and fitting, Appendices~\ref{app:metrics} and~\ref{app:baselines} define metrics and controls, and Appendix~\ref{app:statistics} reports paired counts and bootstrap intervals.

\section{Results}
\label{sec:results}
\subsection{Action selection across physical tasks}
D-JEPA reaches 87.89\% success on PushT and 93.75\% on Reacher under identical candidate availability, outperforming pretrained predictive models and matched adaptation controls (Table~\ref{tab:main}). PushT benefits from calibrated composition; Reacher improves through relational selection. In Granular manipulation, the largest gain occurs at the strict goal threshold, demonstrating improved thresholded attainment. Figure~\ref{fig:additionaldynamics} connects these gains to articulated motion and distributed-particle control; Appendix~\ref{app:statistics} reports paired uncertainty and native task costs. Across four RoboTwin tasks, D-JEPA raises average native VLA success from 61.72\% to 76.76\%; on two physical PiPER tasks, success rises from 64.0\% to 81.0\% with the same V-JEPA~2-AC planner (Table~\ref{tab:robotwin}). Relational alignment outperforms fixed correction and scalar confidence gating, and the preservation gate provides an additional benefit (Table~\ref{tab:robotwincontrols}). Figure~\ref{fig:robotwincase} shows grasping behavior. In seven focused driving scenes, mean PDMS rises from 57.36 to 95.34 under shared candidates (Table~\ref{tab:driving}). Figure~\ref{fig:drivingcase} shows how the aligned trajectory maintains more clearance as traffic evolves.

\subsection{Alignment under shape and visual changes}
D-JEPA improves every held-out PushObj geometry and every evaluated appearance condition (Table~\ref{tab:generalization}), with object-colour change yielding the largest appearance gain. Figure~\ref{fig:generalizationvisual} shows positive net gains throughout; Table~\ref{tab:shiftcounts} gives paired counts and Figure~\ref{fig:generalizationcases} shows executions. These evaluations measure transfer to unseen shapes and performance on fresh starts within the appearance families used for task-local fitting.

\subsection{Ablations of alignment and predictive evidence}
Relational alignment is the strongest individual intervention in the matched ablation (Table~\ref{tab:ablation}). Predictive adaptation contributes through calibrated composition, improving PushT confirmation success (Table~\ref{tab:independentmodules}). Under the same relational architecture on the 128-start mechanism population, dual LeWM and TD-JEPA evidence reaches 93.75\%, compared with 84.38\% and 83.59\% from either source alone. The same relational operator incorporates four predictive geometries and achieves the success rate of calibrated four-source fusion. Figure~\ref{fig:componentsvisual} connects the ablations to uncertainty and deployment cost. D-JEPA improves as candidate availability increases (Appendix~\ref{app:budget}), consistent with learning to exploit relations among a richer set of plausible futures.

\subsection{Realizing decisions in future geometry}
Exact ordinal realization encodes the learned relational order in future geometry: native goal distance reproduces aligned decisions at comparable latency (Table~\ref{tab:deployment}). On Reacher, temporal transport learns signed updates across five future steps within the prescribed bound (Appendix~\ref{app:representation}; Figure~\ref{fig:representationdetail}). These mechanisms expose decision structure through future-latent outputs; Figure~\ref{fig:qualitative} illustrates the physical consequence of the aligned choice.

\section{Conclusion}
We introduced D-JEPA, a decision-aligned latent world model motivated by a simple observation: predictive geometry can remain globally informative while assigning misleading goal distances to the few candidate futures that determine behavior. D-JEPA learns these decision-relevant relations from executed outcomes through bounded set-wise alignment. Restricted predictor adaptation and a shared ordinal interface extend the core relation operator across complementary predictive geometries. The learned decision structure can be realized directly in JEPA-compatible future representations, recovering aligned action choices through native latent distance. Evaluations spanning latent control, manipulation, pretrained action models, physical robots and focused driving scenes demonstrate the effectiveness of decision alignment, with further gains under geometric and visual changes. Together, these findings connect predictive world models to effective control by aligning latent geometry with realized action consequences.
\label{sec:conclusion}

\FloatBarrier
\label{sec:mainend}
\clearpage
\ifdefined\djepaauthorversion
\else
\section*{AI Use Statement}
Codex was used to assist with manuscript writing and language polishing.
\fi
\phantomsection\label{sec:references}
\bibliographystyle{iclr2027_conference}
\bibliography{references}
\clearpage
\appendix
\section{Protocols and Implementation}
\label{app:implementation}
\subsection{Evaluation units and candidate construction}
A start identifies a restored environment state, an observation context and a goal. A candidate identifies a complete stored action sequence for that start. Model-specific scores are aligned by persistent candidate IDs, and simulator outcomes are joined to the chosen IDs. Within a comparison, methods share starts, goals, candidate availability and execution duration. Statistical units follow the task protocol: starts define decision instances, while repeated conditions and visualizations remain associated with their originating starts. Task populations are reported separately.

The PushT confirmation cohort targets decisions with nontrivial candidate boundaries. Its pre-established eligibility rule specifies success/failure candidate structure and a sufficiently small pretrained-score boundary gap. The 256 confirmation starts are drawn from identities disjoint from fitting, calibration and previous development; every compared model evaluates this cohort.

\subsection{Relational and predictive learning}
The $386\!\to\!64$ candidate encoder uses LayerNorm and GELU. Each of the two pre-normalized Transformer layers has four attention heads, a 128-dimensional feed-forward block and zero dropout. The rank-eight output is bounded by 0.2 and initialized to reproduce the base score. The core PushT loss uses temperature 0.05, success/failure margin 0.02, local pair weight 0.25 and squared-correction weight 0.1. The decision subset contains 16 candidates. The calibrated fusion coefficient is 0.42; the recorded relational gate threshold is $-0.0316252634$. The exact-realization implementation rounds the gate advantage to four decimal places before its strict threshold comparison to reproduce the released numerical decision rule.

Fusion is fitted on the designated training pool over a 101-point grid, with five-fold diagnostics; relational weights are fitted on the fitting subset, and the checkpoint and gate are selected on a 64-start calibration subset. The recorded relational configuration runs 1,000 AdamW updates with eight starts per batch, learning rate $3\times10^{-4}$, weight decay $10^{-4}$, gradient clipping at 1.0, and calibration evaluation every 50 updates. Fitting, checkpoint selection and gate calibration use their designated subsets; the 256-start confirmation cohort provides the final evaluation. Persistent-ID tie breaking is retained in scoring, ordinal conversion and realization.

Only the final predictor block and prediction projection receive gradients. The loss preserves both the teacher trajectory and score relationships outside the targeted decision subset. The implementation combines a non-tail distributional KL penalty with a score-difference consistency term, alongside latent MSE and success/failure boundary ordering. The recorded configuration uses 384 fitting starts, 128 calibration starts, 600 updates, learning rate $2\times10^{-6}$, weight decay $10^{-3}$ and gradient clipping at 0.25. Boundary, rank-preservation and latent-preservation weights are 0.5, 0.25 and 1.0. Its standalone readout minimizes the adapted predictor's native goal cost. The composition uses the recorded predictive gate threshold $0.0065339543$ on the calibrated relational-default/predictive-proposal score comparison.

\subsection{Task-specific alignment and temporal transport}
The unseen-shape and visual-change instances use visual, proprioceptive and native-joint within-set ranks as three input coordinates. The encoder is $3\!\to\!64$; set depth, attention heads, rank-eight output and correction bound match the relational construction. These instances use full-list success mass, local ordering with weight 0.25 and correction preservation with weight 0.01. The base is calibrated visual/proprioceptive rank fusion. Fitting and calibration are separate for each task-local checkpoint.

The transport head takes five scalars per candidate and time step: LeWM rank, TD-JEPA rank, inferred relational rank, gate state and normalized time. A 32-dimensional GELU hidden layer maps them to a coefficient bounded by 0.1. Training orders candidates by success first, then physical task cost, with candidate ID breaking ties. The best candidate under this order supplies a cross-entropy target for both the refined score and the transported native criterion. Both also receive all-pairs softplus ordering losses with weight 0.25; a future-preservation MSE has weight 0.02. The temperature is 0.1. Physical labels define the training targets; at deployment, the head receives source ranks, inferred relational rank, gate state and time.

Transport training uses 384 starts and 1,500 updates. The selected 1,200-update snapshot provides representation measurements on 128 calibration starts. Physical success is evaluated on a separate 128-start test population using the relational score to select actions.

\subsection{Expanded populations and timing}
Core evaluations use 256 independent PushT starts, 128 Reacher starts and 64 Granular starts, with a separate 128-start PushT mechanism population. RoboTwin includes 128 starts per task, PiPER has 50 paired trials per task, and the seven driving scenes come from three source logs with 32 trajectories per scene. Core model comparisons share candidate availability; task-local controls use the same fitted predictive evidence. Exact demonstrations are excluded from deployable pools where specified by the task protocol.

PushObj uses 512 fitting starts and 256 calibration starts on shapes T, L, Z and +. The independent shape confirmation combines two 150-start confirmation sets, with 100 starts each for I, small T and square. The earlier 150-start development set is excluded. Appearance learning uses 126 fitting and 63 calibration base starts under seven known conditions. The resulting checkpoint is evaluated on 50 new base starts under each condition. Each selected sequence executes all 25 controls; the exact demonstration sequence is excluded from the 63 deployable candidates.

Table~\ref{tab:deployment} uses batch size 16, two warm-up repeats and ten timed repeats over the same 256 label-free PushT inputs. Latency measures end-to-end model inference per start, including the configured predictive and relational paths. The corresponding checkpoint-file counts are three, four and one; their recorded logical package counts are four, five and one. A single packaged artifact can include multiple predictive backbones.

\subsection{Metric definitions and statistical units}
\label{app:metrics}
For start $j$, let $i_j$ be the selected candidate and $y_{j,i_j}$ its executed binary outcome. Success is $100n^{-1}\sum_{j=1}^{n}y_{j,i_j}$. With $G$ paired gains and $L$ paired losses, the improvement is $100(G-L)/n$ percentage points; shared successes and failures remain in the denominator. Shape groups use distinct starts, whereas the seven appearance conditions reuse the same 50 base identities. PiPER reports an equal-weight average of two task success rates. RoboTwin pools four equally sized task populations. Driving reports the equal-weight mean of seven scene-level PDMS values on a 0--100 scale, rather than a binary success rate.

Figure~\ref{fig:mismatch} compares predicted latent goal costs with latent goal costs computed from realized observations using the same model. For shortlist size $k$, each model retains its $k$ lowest predicted-cost candidates within each start. The within-start statistic averages 96 separate Spearman coefficients; the pooled statistic concatenates the shortlisted predicted/realized cost pairs before computing one coefficient. Table~\ref{tab:diagnostic} reports both definitions. Neither statistic is a correlation with binary success labels.

Granular uses all particle positions $P$ and goal positions $Q$:
\begin{equation}
 \operatorname{CD}(P,Q)=\frac{1}{|P|}\sum_{p\in P}\min_{q\in Q}\|p-q\|_2
 +\frac{1}{|Q|}\sum_{q\in Q}\min_{p\in P}\|q-p\|_2.
 \label{eq:chamfer}
\end{equation}
The two directed means are added, using Euclidean rather than squared distances. Main and strict attainment apply thresholds of 0.18 and 0.12 to the same quantity. Physical stacking requires the released upper cube to remain supported by the lower cube after arm withdrawal for the prescribed observation interval; grasping or lifting alone is not success.

Single-method success intervals in Figures~\ref{fig:componentsvisual} and~\ref{fig:granularmetrics} are Wilson 95\% intervals over starts. Paired differences in Table~\ref{tab:paired} use 10,000 bootstrap resamples of whole paired starts, preserving each pair's outcomes. Candidate-budget bands in Figure~\ref{fig:budgetorder} instead span the minimum and maximum over 16 shared subset seeds. Latency in Table~\ref{tab:deployment} is median end-to-end inference time per start under the common batch/warm-up protocol above, excluding simulator execution.

\subsection{Baseline mechanisms and comparison identities}
\label{app:baselines}
The first block of Table~\ref{tab:main} evaluates native LeWM, TD-JEPA, JEPA-WM and DINO-WM planning criteria on matched candidate actions, scoring futures in each model's own predictive space. Publication identities credit the source models and objectives; the values use this paper's task populations. The second block implements optimization or stochastic-planning controls on the shared JEPA decision problem.

\textit{Gradient-conflict and variational controls.}
The PEGrad-inspired control~\citep{peri2025pegrad} keeps prediction as its primary gradient and difference/ranking supervision as an auxiliary gradient. A negative inner product triggers projection of the conflicting auxiliary component; the remaining norm is capped relative to the primary gradient. The Var-JEPA-inspired control~\citep{gogl2026varjepa} introduces an action-conditioned trajectory code, a training-time posterior conditioned on target futures, and a zero-initialized residual decoder. Reconstruction, prior/posterior regularization and decision losses train the model; inference uses the prior. Conflict-safe variational adaptation combines that residual with conflict-only projection of the reconstruction/KL auxiliary gradient. It is a combined control rather than an additional pretrained backbone.

\textit{Expected-value and tail-risk planning.}
The stochastic controls sample 16 prior futures per candidate and compute their terminal latent goal costs. Expected-value planning minimizes the sample mean. Empirical upper-tail CVaR minimizes the mean of the largest $\lceil(1-0.9)16\rceil=2$ sampled costs. These controls instantiate sampling-based planning~\citep{enwerem2026vnb}; the CVaR objective is related to the risk-sensitive criterion discussed by~\citet{ni2024cvar}. Posterior information is confined to training. Persistent candidate IDs resolve deployment-score ties.

\textit{Fusion and task-specific controls.}
Four-geometry fusion combines calibrated ordinal evidence from LeWM, TD-JEPA, JEPA-WM and DINO-WM without a learned set-wise correction. Shape and appearance fusion instead combines visual and proprioceptive ranks from the same task predictor. RoboTwin's fixed correction, scalar confidence gate and preservation ablation isolate the mechanisms in Table~\ref{tab:robotwincontrols}. The native VLA, V-JEPA~2-AC and Drive-JEPA controls retain their action-producing foundations and shared candidate/search budgets. Candidate Oracles use executed outcomes to quantify available headroom; learned selectors score candidates before execution.

\subsection{Notation and inference correspondence}
\label{app:notation}
In Section~\ref{sec:method}, $K>1$ is the candidate count, $H$ the prediction horizon, $D$ the latent dimension and $n$ the evaluation-start count. Superscripts $L,T,J,D$ identify LeWM, TD-JEPA, JEPA-WM and DINO-WM, while the unadorned $T$ in the loss is a temperature. Source rank $r_i^m\in[0,1]$ is normalized, with smaller values preferred. The final strict rank $\pi_i\in\{1,\ldots,K\}$ includes the calibrated gate before ordinal realization. Rank recovery measures ordering agreement, selected-action agreement measures decision agreement, and success measures executed goal attainment. Appendix~\ref{app:representation} derives the native-distance identity; Appendix~\ref{app:methoddetails} gives the full objectives and task interfaces.

\section{Decision-Alignment Objectives and Task Interfaces}
\label{app:methoddetails}
\subsection{Complete-set evidence and learning objectives}
The dual-model token contains two separately normalized 192-dimensional descriptors and two within-set ranks, giving 386 inputs. Each goal-relative difference is computed within its own backbone. Persistent action IDs break rank ties. The base fusion coefficient $\alpha$ is fitted on training/calibration data. A shared encoder maps the token to 64 dimensions, followed by two four-head, dropout-free Transformer layers without candidate-order positional encodings. Permuting input candidates therefore permutes their outputs. The correction head has dimensions $64\!\to\!8\!\to\!1$, and its final projection is initialized to zero.

For training sets containing successful and unsuccessful candidates, the full-list objective rewards any successful alternative rather than a single reference trajectory:
\begin{equation}
 \mathcal L_{\rm success}=-\log\sum_{i:y_i=1}p_i.
 \label{eq:success}
\end{equation}
The dynamically refreshed subset $\mathcal T_k$ contains the $k$ lowest current scores. Its success/failure pairs receive
\begin{equation}
 \mathcal L_{\rm local}=
 \underset{\substack{i\in\mathcal T_k:y_i=1\\j\in\mathcal T_k:y_j=0}}{\operatorname{mean}}
 \softplus\!\left(\frac{\mu+s_i-s_j}{T}\right).
 \label{eq:local}
\end{equation}
If the subset contains one class, pairs come from the complete set. Equation~\ref{eq:relational} combines these terms with squared-correction regularization. The subset concentrates supervision; inference still processes all $K$ candidates. With $i_b=\argmin_i b_i$ and $i_s=\argmin_i s_i$, the relational winner is admitted when $b_{i_b}-s_{i_s}>\tau_R$; otherwise the base action is retained. The checkpoint, fusion and gate are fixed on calibration data.

Predictive plasticity trains only the final predictor block and prediction projection. Its objective is
\begin{equation}
 \mathcal L_P=\mathcal L_{\rm success}
 +\lambda_B\mathcal L_{\rm boundary}
 +\lambda_K\mathcal L_{\rm rank}
 +\lambda_Z\mathcal L_{\rm latent}.
 \label{eq:plastic}
\end{equation}
Boundary pairs follow the fixed reference ordering. Rank preservation outside the boundary and latent consistency discourage unnecessary changes elsewhere in the predicted trajectory. The standalone score is the adapted future's native goal distance. In composition, the relational decision is the default. Its winner $i_R$ is replaced by predictor winner $i_P$ when $s^R_{i_R}-c^P_{i_P}>\tau_P$, using the recorded calibrated score construction and threshold.

\subsection{Predictive geometries and action-producing models}
The four-geometry PushT token is $[d_i^L;d_i^T;r_i^L;r_i^T;r_i^J;r_i^D]\in\mathbb R^{388}$, where $J,D$ denote JEPA-WM and DINO-WM. The expanded operator is trained for this input and refines its recorded JEPA-WM base score. Dense directional information stays within each predictive space; ordinal coordinates supply a common comparison interface. Shape and appearance experiments instead use visual, proprioceptive and native-joint ranks, with a $3\!\to\!64$ encoder and the same two-layer set operator and rank-eight head. Separate task-local checkpoints refine calibrated visual/proprioceptive fusion.

In RoboTwin, the pretrained VLA maps language, synchronized multi-view RGB and proprioception to a native bimanual action chunk and future token. One native future and 16 bounded spatial/temporal variants form the set. The operator compares predictive, visual, proprioceptive and temporal evidence. A learned preservation gate retains an already coherent native future; the selected relation is written back to the JEPA-compatible future/action interface consumed by the VLA decoder. Appendix~\ref{app:robotinterfaces} details observation conventions, scene-conditioned grasp proposals and their execution interface.

Drive-JEPA supplies 32 trajectories per scene. Each token combines a 256-dimensional query, native score and rank, and 24-dimensional trajectory geometry. A projection and two-layer, four-head relational encoder produce bounded corrections relative to the native winner, whose correction is fixed at zero. Supervised collision and time-to-collision readouts can penalize a switch; the native trajectory is retained unless the corrected margin supports replacement. On PiPER, D-JEPA similarly re-ranks shared candidates from a V-JEPA~2 action-conditioned planner. The predictor, search budget and low-level controller are unchanged. Appendices~\ref{app:drivinginterface} and~\ref{app:physicalinterface} specify the driving readouts and physical-robot data/control boundary.

\subsection{Future-latent realization}
Reacher's coefficient network receives source ranks, inferred relational rank, gate state and normalized time:
$\beta_{i,t}=\rho\tanh f_\phi(r_i^L,r_i^T,r_i^R,q,t/H)$.
The $5\!\to\!32\!\to\!1$ network has a zero-initialized output, five future steps and $\rho=0.1$. Equation~\ref{eq:transport} updates only the corresponding action and time index. Reacher evaluation measures executed success through relational selection and learned future displacement through the calibration diagnostic (Appendix~\ref{app:implementation}).

For PushT, persistent IDs make final gated ranks strict before exact realization. A deterministic unit-RMS direction handles zero terminal displacement. Terminal realization encodes the goal- and candidate-set-conditioned order in the radius while retaining the original direction and preceding future steps. Rank identity holds at deployed precision. The self-contained checkpoint embeds both predictive models and relational computation, exposing the aligned decision directly through native latent distance.

\section{Complete Numerical Details}
\label{app:statistics}
\begin{table*}[t]
\centering
\caption{\textbf{Language-conditioned bimanual manipulation in RoboTwin.} Success rates under each task's full completion criterion, with 128 independent test starts per task. All methods share the pretrained VLA, observations, candidate identities and execution horizon.}
\label{tab:robotwincontrols}
\fontsize{7.0}{9.0}\selectfont
\setlength{\tabcolsep}{2.2pt}
\renewcommand{\arraystretch}{1.10}
\begin{tabularx}{\textwidth}{>{\hsize=1.6\hsize\linewidth=\hsize}L*{5}{>{\hsize=.88\hsize\linewidth=\hsize}C}}
\toprule
\rowcolor{headerbg}
\textbf{Method} & \makecell{\textbf{Grab}\\\textbf{roller}} & \makecell{\textbf{Bread}\\\textbf{$\rightarrow$ skillet}} & \makecell{\textbf{Object}\\\textbf{$\rightarrow$ cabinet}} & \makecell{\textbf{Block}\\\textbf{handover}} & \makecell{\textbf{Overall}\\$n=512$}\\
\midrule
\rowcolor{referencebg}Native VLA & 71.09 & 64.06 & 59.38 & 52.34 & 61.72\\
Fixed correction & 75.00 & 68.75 & 64.84 & 58.59 & 66.80\\
Scalar gate & 75.78 & 67.97 & 63.28 & 57.81 & 66.21\\

\rowcolor{oursbg}\textbf{D-JEPA}\enspace\paperbadge{Full model} & \bestcell{83.59} & \bestcell{78.91} & \bestcell{75.00} & \bestcell{69.53} & \bestcell{76.76}\\
\midrule
Oracle & 89.06 & 85.16 & 82.03 & 77.34 & 83.40\\
\bottomrule
\end{tabularx}
\tablenote{Without preservation, overall success is 71.48\%. Entries report success percentages over 128 held-out starts per task. D-JEPA improves overall success by 15.04 points; Oracle is the paired candidate-outcome ceiling.}
\end{table*}

\begin{table*}[t]
\centering
\caption{\textbf{Paired success differences on the core formal populations.} A gain is D-JEPA success with baseline failure; a loss is the reverse. Intervals resample paired starts 10,000 times (seed 3072). Differences and intervals are percentage points.}
\label{tab:paired}
\tablefont
\begin{tabularx}{\textwidth}{L L r r r r}
\toprule
\rowcolor{headerbg}
\textbf{Task / configuration} & \textbf{Reference} & \textbf{$\Delta$ (pp)} & \textbf{95\% interval} & \textbf{Gains} & \textbf{Losses}\\
\midrule
PushT composition, $n=256$ & LeWM & \textbf{+4.30} & $[+0.39,+8.20]$ & 19 & 8\\
PushT composition, $n=256$ & TD-JEPA & \textbf{+10.94} & $[+5.86,+16.02]$ & 37 & 9\\
\midrule Reacher relational, $n=128$ & LeWM & \textbf{+7.03} & $[+0.78,+14.06]$ & 15 & 6\\
Reacher relational, $n=128$ & TD-JEPA & \textbf{+25.00} & $[+17.19,+32.81]$ & 33 & 1\\
\bottomrule
\end{tabularx}
\end{table*}

\begin{table}[t]
\begin{minipage}[t]{0.51\linewidth}\vspace{0pt}
\captionof{table}{\textbf{Independent PushT modules.} Shared confirmation population.}
\label{tab:independentmodules}
{\tablefont\setlength{\tabcolsep}{3pt}
\begin{tabularx}{\linewidth}{Lrr}
\toprule
\rowcolor{headerbg}\textbf{Configuration} & \textbf{Successes} & \textbf{\% $\uparrow$}\\
\midrule
TD-JEPA & 197/256 & 76.95\\
LeWM & 214/256 & 83.59\\
Predictive plasticity & 204/256 & 79.69\\
Relational alignment & 223/256 & 87.11\\
\rowcolor{oursbg}Calibrated composition & \textbf{225/256} & \textbf{87.89}\\
\bottomrule
\end{tabularx}}
\tablenote{Native adapted-future selection and calibrated composition are distinct readouts. All rows use the same starts and candidate pools as Table~\ref{tab:main}.}
\end{minipage}\hfill
\begin{minipage}[t]{0.46\linewidth}\vspace{0pt}
\captionof{table}{\textbf{Reacher success and cost.} Shared 128-start population.}
\label{tab:reachercost}
{\fontsize{6.8}{9.0}\selectfont\setlength{\tabcolsep}{2pt}
\begin{tabularx}{\linewidth}{Lrr}
\toprule
\rowcolor{headerbg}\textbf{Method} & \textbf{Successes (\%)} & \textbf{Cost $\downarrow$}\\
\midrule
TD-JEPA & 88/128 (68.75) & 0.0515\\
LeWM & 111/128 (86.72) & 0.0302\\
\rowcolor{oursbg}D-JEPA & \textbf{120/128 (93.75)} & \textbf{0.0245}\\
\bottomrule
\end{tabularx}}
\tablenote{Parentheses give success rates. The aligned choice improves goal attainment and terminal task cost under the same candidate actions. Relational selection supplies the executed actions; Appendix~\ref{app:representation} reports temporal-transport measurements.}
\end{minipage}
\end{table}

\subsection{Decision-local ranking diagnostics}
\label{app:distanceconflict}
Figure~\ref{fig:mismatch} combines executed-outcome diagnostics with predicted--realized latent-cost correlations. For each audit start and model, let $d_{\mathrm{s}}$ and $d_{\mathrm{f}}$ be the minimum predicted goal RMS distances among successful and failing deployable candidates. The normalized gap is $(d_{\mathrm{s}}-d_{\mathrm{f}})/(d_{\mathrm{s}}+d_{\mathrm{f}})$; a positive value means that the lowest-distance failing candidate outranks every successful alternative. All 96 audit starts contain both outcomes, and each model has eight positive-gap starts. Panel (b) includes every start, with median and interquartile summaries. The separate recorded illustration uses the same LeWM model, start and candidate pool for both actions: candidate A has distance 0.2027 and fails, whereas candidate B has distance 0.2167 and succeeds.

For panel (c), we retain each model's $K$ lowest predicted-cost candidates and form every success--failure pair within the shortlist. An inversion occurs when the failing candidate has strictly lower predicted cost; tied costs are not inversions. We average each start's inversion fraction equally across starts containing both outcomes. At $K=63,16,4$, eligible-start counts are $96,96,34$ for LeWM and $96,95,29$ for TD-JEPA. Their corresponding mean inversion percentages are $3.95,20.44,49.02$ and $3.87,19.44,38.22$. Each shortlist defines its own mixed-outcome population; these percentages measure pair ordering, not the planner's failure rate. All diagnostics exclude the expert candidate and use the existing recorded executions.

Table~\ref{tab:diagnostic} reports both correlation definitions from Figure~\ref{fig:mismatch}, including intermediate shortlist sizes. The two blocks follow Appendix~\ref{app:metrics}. Shortlists are nested within a model, so the columns are not independent samples.

\begin{table}[t]
\begin{minipage}[t]{.49\linewidth}\vspace{0pt}
\captionof{table}{\textbf{Decision-local ranking.} Within-start and pooled Spearman correlation.}
\label{tab:diagnostic}
{\fontsize{6.8}{9.0}\selectfont\setlength{\tabcolsep}{2.2pt}
\begin{tabularx}{\linewidth}{p{49pt}*{5}{>{\raggedleft\arraybackslash}X}}
\toprule
\rowcolor{headerbg}\textbf{Model} & \textbf{63} & \textbf{32} & \textbf{16} & \textbf{8} & \textbf{4}\\
\midrule
\rowcolor{sectiongray}\multicolumn{6}{l}{\textit{Within-start mean}}\\
LeWM & .895 & .842 & .621 & .339 & .108\\
TD-JEPA & .798 & .773 & .560 & .353 & .126\\
\rowcolor{sectiongray}\multicolumn{6}{l}{\textit{Pooled}}\\
LeWM & .931 & .894 & .774 & .610 & .519\\
TD-JEPA & .850 & .820 & .683 & .503 & .357\\
\bottomrule
\end{tabularx}}
\tablenote{Columns are predicted-cost shortlist sizes over the same 96 starts. Smaller shortlists focus on actions closer to execution.}
\end{minipage}\hfill
\begin{minipage}[t]{.48\linewidth}\vspace{0pt}
\captionof{table}{\textbf{Candidate-budget success.} Shared nested candidate subsets.}
\label{tab:budget}
{\fontsize{6.8}{9.0}\selectfont\setlength{\tabcolsep}{2.2pt}\renewcommand{\arraystretch}{1.50}
\begin{tabularx}{\linewidth}{p{49pt}*{4}{>{\raggedleft\arraybackslash}X}}
\toprule
\rowcolor{headerbg}\textbf{Model} & \textbf{8} & \textbf{16} & \textbf{32} & \textbf{63}\\
\midrule
LeWM & 80.42 & 81.52 & 82.18 & 83.59\\
TD-JEPA & 77.10 & 77.47 & 77.95 & 76.95\\
\rowcolor{oursbg}\textbf{D-JEPA} & \bestcell{81.96} & \bestcell{83.94} & \bestcell{86.28} & \bestcell{87.11}\\
\bottomrule
\end{tabularx}}
\tablenote{Success (\%) over 256 independent PushT starts, averaged across 16 shared subset seeds. D-JEPA uses relational selection. Every method receives the same nested subsets; weights and thresholds remain fixed. Full-pool values recover the corresponding results in Table~\ref{tab:independentmodules}.}
\end{minipage}
\end{table}

\subsection{Granular and generalization details}
Granular evaluates the full particle distribution with symmetric summed Chamfer distance and reports both thresholded attainment and continuous costs. The main relational instance improves attainment at the registered standard and strict thresholds of 0.18 and 0.12 relative to native selection. Table~\ref{tab:main} reports mean costs alongside attainment; Figure~\ref{fig:granularmetrics} shows the complete paired cost distribution.

\begin{table}[t]
\centering
\caption{\textbf{Paired outcomes under geometric and visual change.} D-JEPA versus calibrated fusion. Counts retain all starts, including shared successes and failures.}
\label{tab:shiftcounts}
\tablefont
\begin{tabularx}{\columnwidth}{>{\hsize=1.65\hsize\linewidth=\hsize\raggedright\arraybackslash}X*{5}{>{\hsize=.87\hsize\linewidth=\hsize\centering\arraybackslash}X}}
\toprule
\rowcolor{headerbg}
\textbf{Condition} & \textbf{Fusion} & \textbf{D-JEPA} & \textbf{Gains} & \textbf{Losses} & \textbf{Net}\\
\midrule
\rowcolor{sectiongray}\multicolumn{6}{l}{\textit{Held-out shapes: 100 independent starts per geometry}}\\
I shape & 37/100 & 52/100 & 21 & 6 & \textbf{+15}\\
Small T & 61/100 & 66/100 & 13 & 8 & \textbf{+5}\\
Square & 36/100 & 43/100 & 14 & 7 & \textbf{+7}\\
\midrule
\rowcolor{sectiongray}\multicolumn{6}{l}{\textit{Appearance: the same 50 new base starts across seven conditions}}\\
Clean & 35/50 & 39/50 & 4 & 0 & \textbf{+4}\\
Blur & 38/50 & 41/50 & 3 & 0 & \textbf{+3}\\
Salt-and-pepper & 36/50 & 41/50 & 6 & 1 & \textbf{+5}\\
Darkening & 37/50 & 41/50 & 4 & 0 & \textbf{+4}\\
Object colour & 15/50 & 27/50 & 15 & 3 & \textbf{+12}\\
Goal-marker colour & 28/50 & 30/50 & 5 & 3 & \textbf{+2}\\
Pusher colour & 31/50 & 37/50 & 6 & 0 & \textbf{+6}\\
\bottomrule
\end{tabularx}
\end{table}

D-JEPA improves over both native selection and calibrated fusion on every held-out geometry, with the largest gain over fusion on the I shape. Each shape contributes a distinct, equally sized population, so the aggregate in Table~\ref{tab:generalization}A weights the three geometries equally. Table~\ref{tab:shiftcounts} further separates recovered successes from regressions instead of repeating the aggregate rates.

\section{Candidate Availability and Representation Analyses}
\label{app:budget}
\subsection{Candidate-budget protocol}
All methods use identical, uniformly sampled nested subsets for budgets 8, 16, 32 and 63, with 16 subset seeds (9201--9216). Evaluation uses the calibrated weights and thresholds, with within-subset ranks and relational outputs recomputed for each uniformly sampled candidate subset. Success labels come from the original executed candidate rollouts. At 63 candidates, each method reproduces all of its released selected IDs. The curves report means and min--max ranges over subset seeds. D-JEPA success increases from 81.96\% with eight candidates to 87.11\% with 63, demonstrating improved use of richer candidate availability.

\subsection{Representation realization}
\label{app:representation}
Figure~\ref{fig:liftingdetail} complements the computational structure in Figure~\ref{fig:lifting} with a geometric view of same-action transport and native-distance realization.

\begin{figure}[t]
\centering
\includegraphics[width=\linewidth]{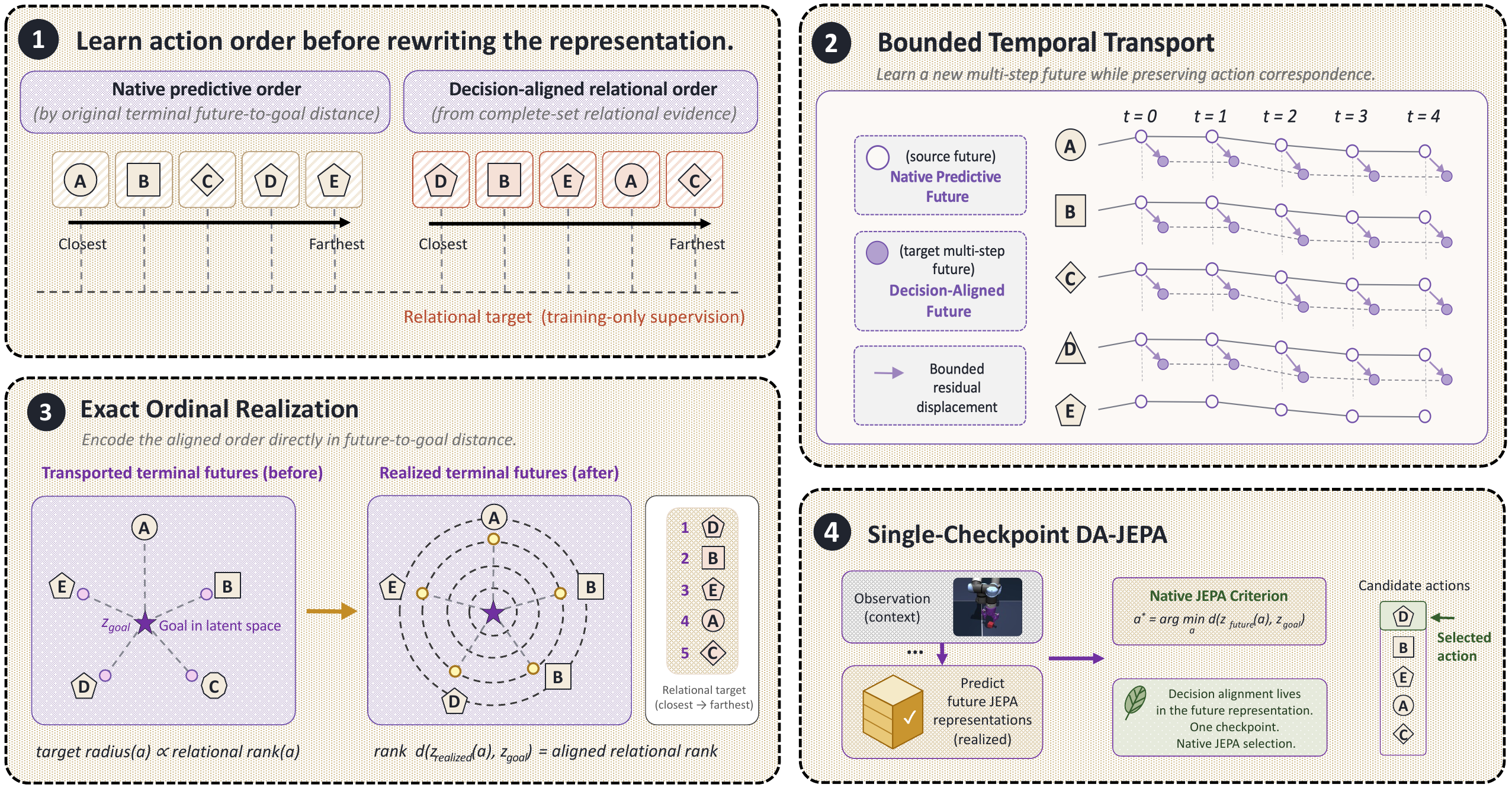}
\caption{\textbf{From aligned action preferences to realized future geometry.} Bounded temporal transport refines future trajectories with matched candidate and time correspondence. Exact ordinal realization places terminal futures at goal-relative radii specified by the final strict ranks. Native goal distance then recovers the aligned choice. The panels illustrate the complementary mechanisms formalized in Equations~\ref{eq:transport} and~\ref{eq:ordinal}; Table~\ref{tab:deployment} verifies the ordinal interface numerically.}
\label{fig:liftingdetail}
\end{figure}

\textit{Ordinal identity.}
Proposition~\ref{prop:exactrealization} proves that Equation~\ref{eq:ordinal} realizes the final strict rank $\pi_i$ through native mean-squared distance $(\pi_i/(K+1))^2$. Persistent IDs resolve score ties before realization; a deterministic unit-RMS direction handles zero terminal displacement. Table~\ref{tab:deployment} verifies complete rank and selected-action recovery at deployed precision. Terminal realization preserves the first four predicted steps of the five-step PushT trajectory.

\textit{Temporal displacement.}
Equation~\ref{eq:transport} bounds the per-time-step L2 update by $\rho$. The observed maximum is 0.0017227774 on the 128-start calibration representation diagnostic. Time-conditioned coefficients can have either sign, allowing different parts of the same candidate trajectory to change in different directions along the prescribed cross-model displacement.

\begin{figure}[t]
\centering
\begin{subfigure}[t]{0.24\linewidth}\centering
\includegraphics[width=\linewidth]{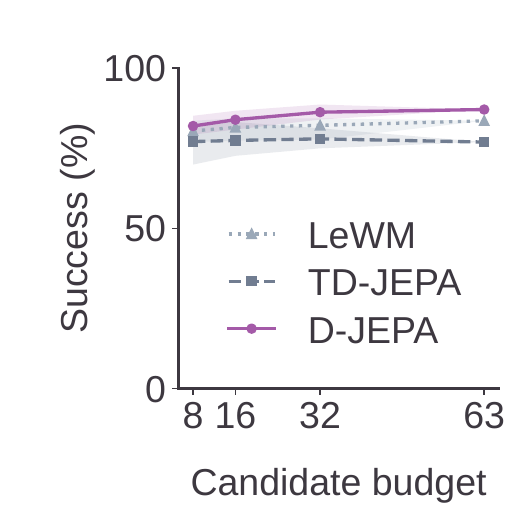}
\caption{Candidate budget.}
\end{subfigure}\hfill
\begin{subfigure}[t]{0.24\linewidth}\centering
\includegraphics[width=\linewidth]{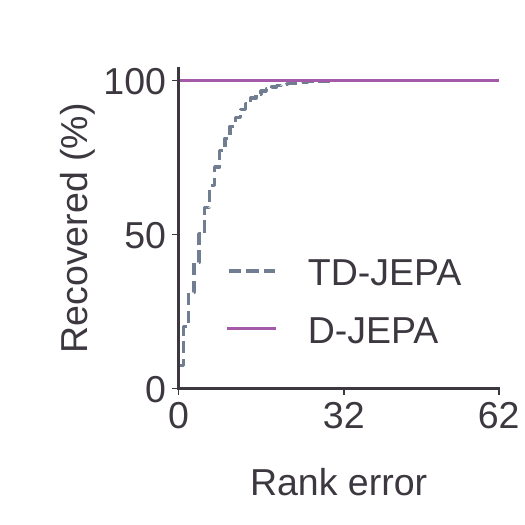}
\caption{Rank recovery.}
\end{subfigure}\hfill
\begin{subfigure}[t]{0.24\linewidth}\centering
\includegraphics[width=\linewidth]{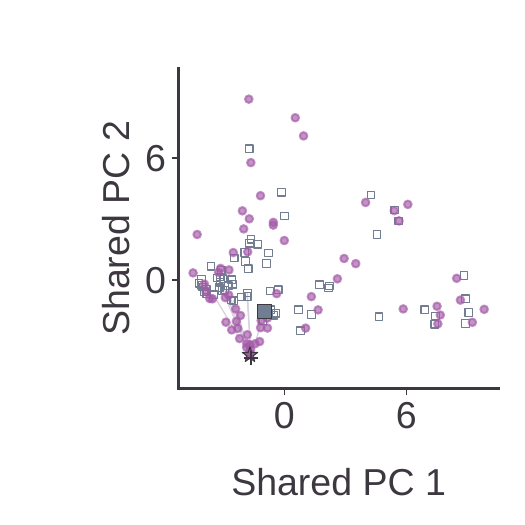}
\caption{Terminal geometry.}
\end{subfigure}\hfill
\begin{subfigure}[t]{0.24\linewidth}\centering
\includegraphics[width=\linewidth]{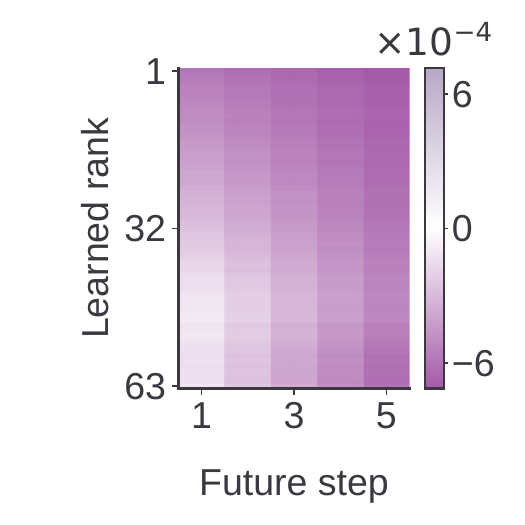}
\caption{Temporal transport.}
\end{subfigure}
\caption{\textbf{Decision learning and its representation mechanisms.} (a) Fixed-weight relational inference with shared candidate subsets; shaded bands span 16 subset seeds. (b) Native recovery of all 16,128 learned ranks over 256 starts. (c) PushT start 176, all 63 candidate endpoints in one joint PCA (47.2\% and 18.4\% explained variance): open grey squares are original predictive futures, purple dots are lifted futures, filled square/star mark selected actions, and the plus marks the goal. The projection visualizes terminal future representations in latent space. (d) Mean signed transport coefficient by learned rank and future step on 128 Reacher calibration starts; the symmetric colour scale is in units of $10^{-4}$.}
\label{fig:budgetorder}
\label{fig:representationdetail}
\end{figure}

\section{Additional Physical Visualizations}
Qualitative figures pair recorded executions under matched starts, goals and horizons. Selected baseline-failure/D-JEPA-success cases illustrate the action consequences, alongside full-population outcomes in the numerical tables.

\begin{figure}[t]
\centering
\begin{subfigure}[t]{0.24\linewidth}\centering
\includegraphics[width=\linewidth]{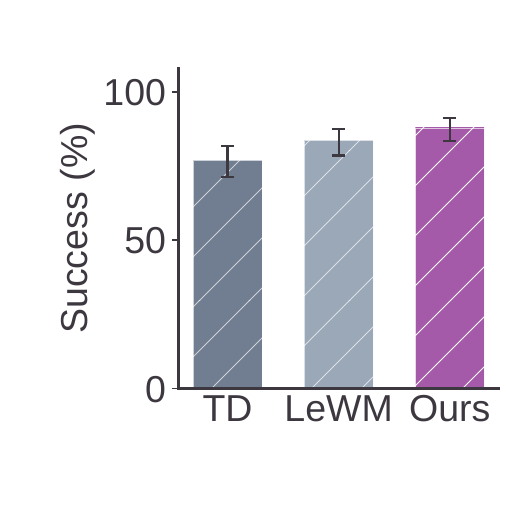}
\caption{PushT success.}
\end{subfigure}\hfill
\begin{subfigure}[t]{0.24\linewidth}\centering
\includegraphics[width=\linewidth]{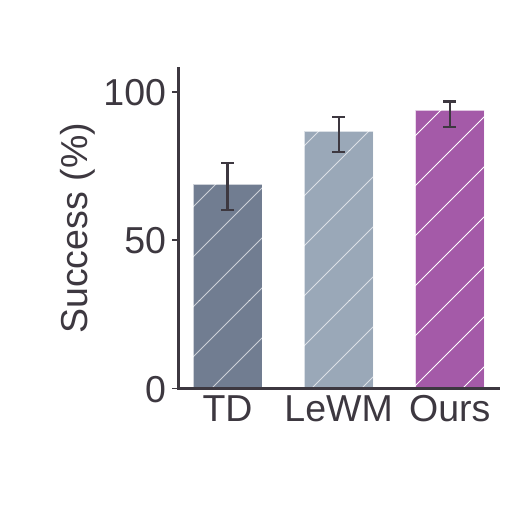}
\caption{Reacher success.}
\end{subfigure}\hfill
\begin{subfigure}[t]{0.24\linewidth}\centering
\includegraphics[width=\linewidth]{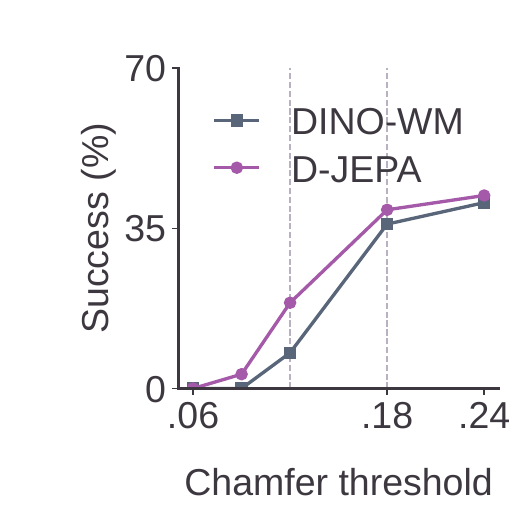}
\caption{Granular attainment.}
\end{subfigure}\hfill
\begin{subfigure}[t]{0.24\linewidth}\centering
\includegraphics[width=\linewidth]{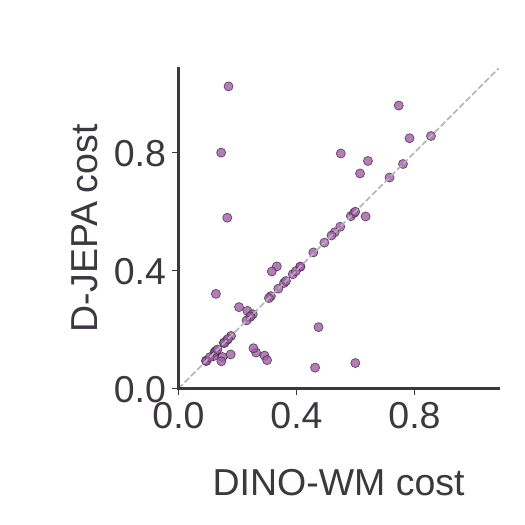}
\caption{Granular costs.}
\end{subfigure}
\caption{\textbf{Executed outcomes across the formal task populations.} (a,b) PushT ($n=256$) and Reacher ($n=128$) success, with Wilson 95\% intervals over starts; TD denotes TD-JEPA and Ours denotes D-JEPA. Configurations match Table~\ref{tab:main}. (c) Granular success across the registered Chamfer thresholds, with strict/main thresholds at 0.12/0.18. (d) All 64 paired Granular costs; points below the diagonal favour D-JEPA. Thresholded attainment and the full cost distribution describe complementary outcomes.}
\label{fig:granularmetrics}
\end{figure}

\begin{figure}[t]
\centering
\begin{subfigure}[t]{0.48\linewidth}
\centering
\includegraphics[width=\linewidth]{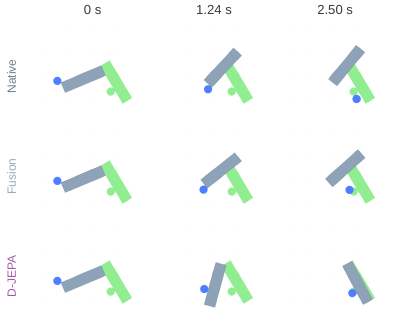}
\caption{Unseen I-shaped object.}
\label{fig:shapequal}
\end{subfigure}\hfill
\begin{subfigure}[t]{0.48\linewidth}
\centering
\includegraphics[width=\linewidth]{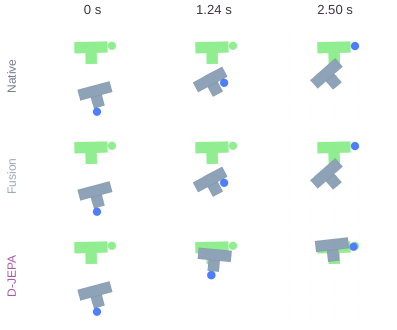}
\caption{Unseen small T-shaped object.}
\label{fig:smalltqual}
\end{subfigure}

\vspace{4pt}
\begin{subfigure}[t]{0.48\linewidth}
\centering
\includegraphics[width=\linewidth]{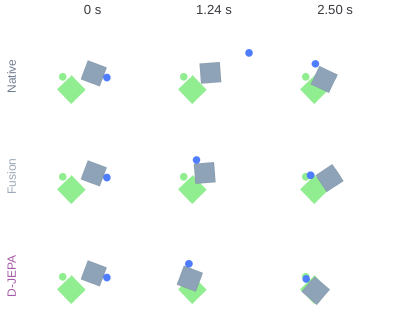}
\caption{Unseen square object.}
\label{fig:squarequal}
\end{subfigure}\hfill
\begin{subfigure}[t]{0.48\linewidth}
\centering
\includegraphics[width=\linewidth]{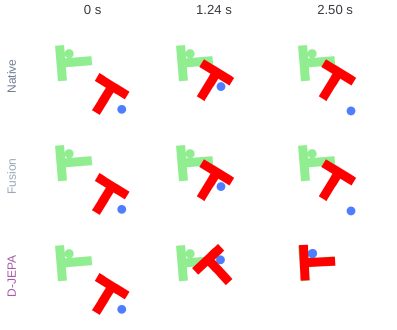}
\caption{Object-colour change.}
\label{fig:shiftqual}
\end{subfigure}
\caption{\textbf{Matched decisions under geometric and visual change.} Native selection, calibrated fusion and D-JEPA share the start, goal and candidate pool in every panel. The frames span the complete selected sequence and retain the measured success or failure outcome.}
\label{fig:generalizationcases}
\end{figure}
Figures~\ref{fig:additionaldynamics} and~\ref{fig:generalizationcases} provide task execution context for the corresponding quantitative experiments. Individual frames, longer sequences, additional appearance cases and synchronized videos are provided as supplementary visualizations.

\section{Task Coverage and Additional Diagnostics}
\label{app:case_sequences}
\textit{Objective preservation.}
A controlled adaptation study identified a mismatch between improvements in global prediction error and action selection. This finding motivated supervision concentrated near the decision boundary, with predictive structure retained elsewhere.

\textit{Appearance evaluation.}
The visual-shift experiment measures decision alignment across changes in blur, noise, illumination and colour. Task-local fitting uses the specified appearance families, followed by evaluation on fresh starts from those families (Table~\ref{tab:generalization}B).

Figure~\ref{fig:visualconditions} shows the seven appearance conditions using the same physical start and goal. The images are recorded context and goal observations supplied to the predictive models.
\begin{figure}[!htbp]
\centering
\includegraphics[width=\linewidth]{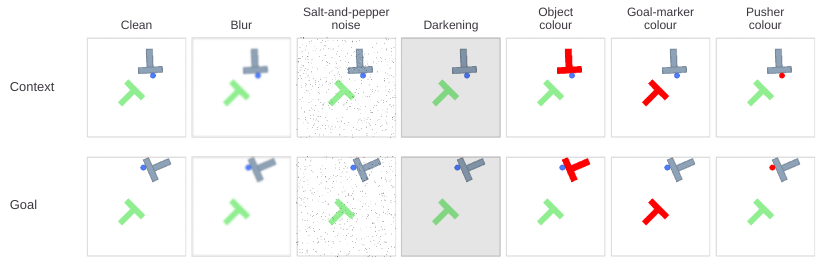}
\caption{\textbf{Seven appearance conditions at a shared physical state.} Columns vary blur, noise, illumination and object, goal-marker or pusher colour relative to the clean input. The upper and lower rows show context and goal observations, respectively. All cells retain their original observed colours and native image detail.}
\label{fig:visualconditions}
\end{figure}

\textit{Cube task coverage.}
The Cube task screen extends the manipulation coverage to a regime with near-saturated success. Figure~\ref{fig:cubecoverage} shows LeWM and TD-JEPA reaching the same goal through different selected action sequences from a shared start.
\begin{figure}[!htbp]
\centering
\includegraphics[width=\linewidth]{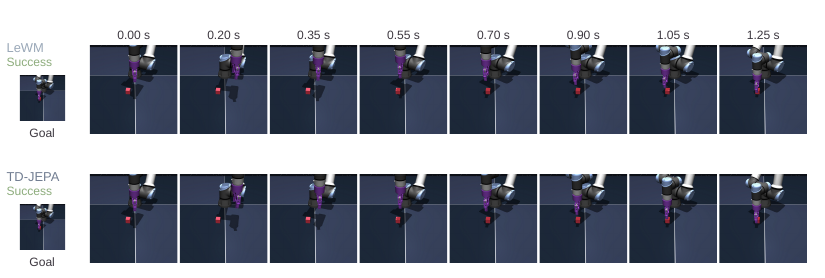}
\caption{\textbf{Cube task coverage with saturated success.} LeWM and TD-JEPA execute their selected 25-control sequences from the same start toward the same goal. Eight shared timestamps cover the full horizon; both baselines succeed. The paired execution illustrates different action choices with the same successful outcome.}
\label{fig:cubecoverage}
\end{figure}

\subsection{Robotic manipulation}
Figures~\ref{fig:robotsequencea} and~\ref{fig:robotsequenceb} examine two additional grasping configurations from complementary viewpoints. The oblique view exposes the approach and grasp, while the overhead view emphasizes the gripper--object geometry and enlarges the terminal state. Both compare saved native and aligned action sequences from an identical restored state. Intermediate frames use matched simulation steps; endpoints retain each sequence's own stopping step. The action construction and execution interface are described in Appendix~\ref{app:robotinterfaces}.

\begin{figure}[!htbp]
\centering
\includegraphics[width=\linewidth]{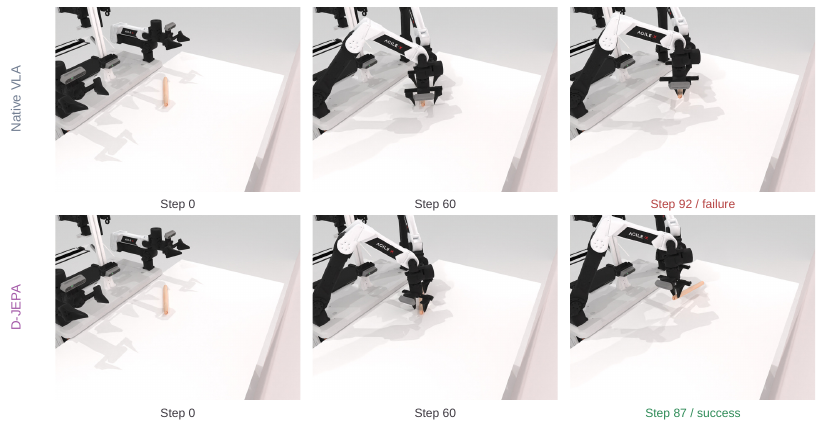}
\caption{\textbf{Approach and grasp from an oblique viewpoint.} The initial state, an intermediate contact configuration and the recorded endpoint reveal different physical outcomes. D-JEPA attains the grasping criterion while native execution remains unsuccessful. Both rows use the same fixed camera and unmodified saved actions; endpoint labels give the respective stopping steps.}
\label{fig:robotsequencea}
\end{figure}

\begin{figure}[!htbp]
\centering
\includegraphics[width=\linewidth]{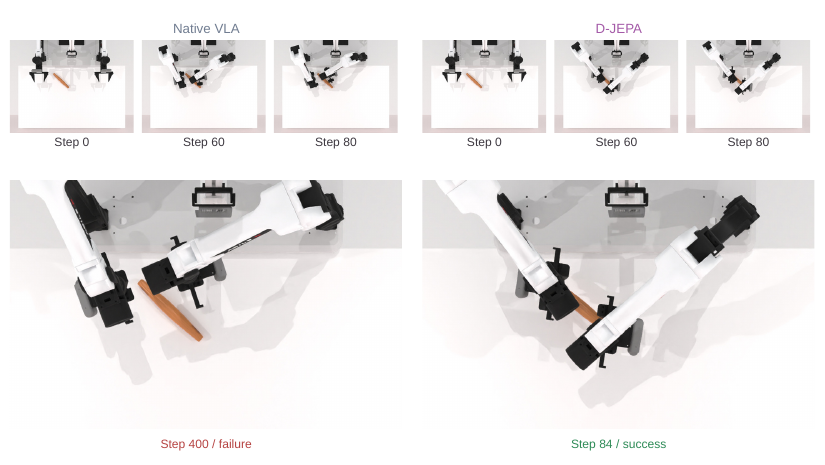}
\caption{\textbf{Contact geometry and terminal outcomes from overhead.} Small matched-step frames provide the approach context; the larger panels use an identical crop to expose the final gripper--object configuration. D-JEPA completes the grasp, while native execution remains unsuccessful through its full stored sequence. Endpoint labels retain the respective stopping steps.}
\label{fig:robotsequenceb}
\end{figure}

\FloatBarrier
\subsection{Autonomous driving}
Figures~\ref{fig:drivesequencecomfort} and~\ref{fig:drivesequenceprogress} complement Figure~\ref{fig:drivingcase} with a turning sequence and a keyframe-focused progress comparison. In each scene, the relational readout selects from the same Drive-JEPA proposals. Logged camera images supply shared scene context, while bird's-eye views show simulated ego motion against recorded traffic. Paired views use identical physical scales and coordinate bounds to reveal the consequences of each selected action.

\begin{figure}[!htbp]
\centering
\includegraphics[width=.92\linewidth]{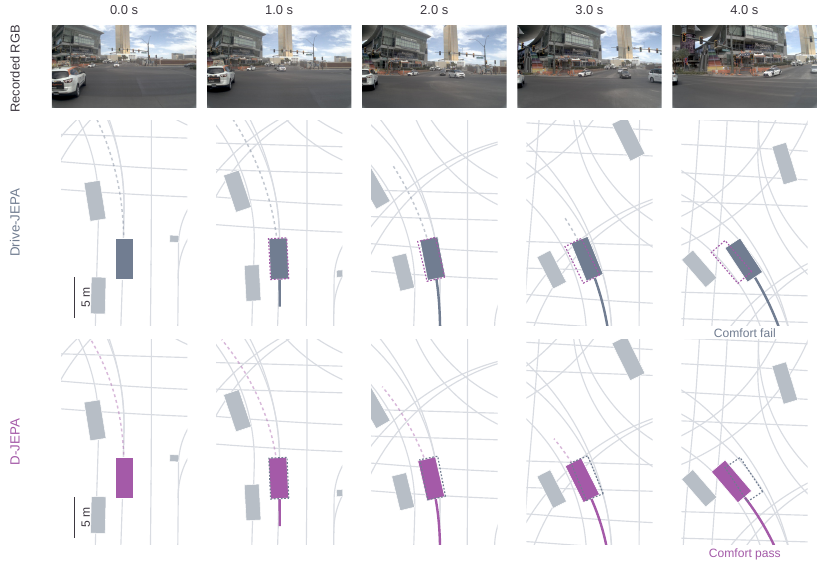}
\caption{\textbf{Aligned turning satisfies the comfort criterion.} Columns sample the four-second trajectory at one-second intervals. Paired windows follow the mean position; dashed footprints locate the alternative ego position. Only D-JEPA passes the trajectory-level comfort criterion, while both methods pass the recorded collision and TTC criteria.}
\label{fig:drivesequencecomfort}
\end{figure}

\begin{figure}[!htbp]
\centering
\includegraphics[width=.92\linewidth]{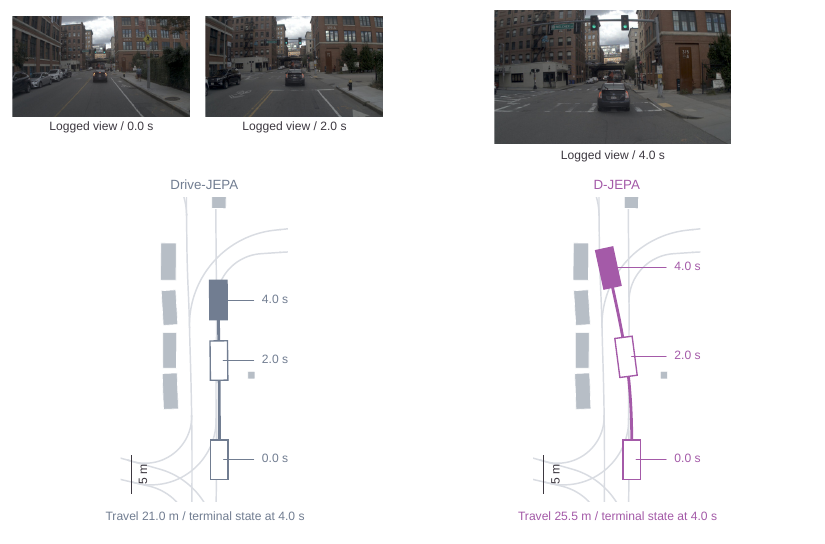}
\caption{\textbf{Keyframes expose different progress through the same traffic scene.} Logged images provide context for the enlarged four-second view. Equal-scale bird's-eye panels show recorded traffic, earlier ego footprints (open) and endpoints (filled). D-JEPA travels farther and attains a higher progress score; both trajectories pass the recorded collision, TTC and comfort criteria.}
\label{fig:drivesequenceprogress}
\end{figure}

\section{Application-Specific Inputs and Execution Interfaces}
\label{app:applications}
The task interfaces instantiate decision alignment at the point where a predictive or action-producing model exposes alternative futures. Physical interaction uses goal-relative predictive evidence; robotic manipulation additionally exposes gripper, object and temporal relations; driving exposes trajectory queries and predicted scores. The shared learning objectives are given in Appendix~\ref{app:methoddetails}. Here we describe how observations become candidates, how alignment affects execution, and which quantities are available before an action is taken. Additional matched execution sequences are shown in Appendix~\ref{app:case_sequences}.

\subsection{Robotic manipulation: observations, proposals and execution}
\label{app:robotinterfaces}
RoboTwin observations contain synchronized RGB views, robot proprioception and a language instruction. The data interface retains four $240\!\times\!320$ RGB views, 14-dimensional joint state/action vectors and 20-dimensional end-effector state/action vectors. Conversion to the execution interface is explicit: each arm contributes a three-dimensional position, a unit quaternion and a gripper value, giving a 16-dimensional bimanual command. Native policy outputs and corrected futures use the same action convention. The demonstration inventory is partitioned by episode identity into training, calibration and offline-validation subsets; these subsets are distinct from simulator evaluation starts.

\textit{Scene-conditioned grasp proposals.}
The bimanual grasping interface used for Figure~\ref{fig:robotwincase} predicts where a demonstration's grasp relation should occur in the current scene. From the initial head-camera image, it extracts the tool's chromatic support and summarizes its image-plane centre, principal direction, extent and two endpoints as nine geometry coordinates. A standardized ridge map predicts the six-dimensional pair of left/right grasp positions. Training demonstrations fit the map; calibration demonstrations select regularization from $\{10^{-4},10^{-3},10^{-2},10^{-1},1\}$. The predicted target span is normalized using training demonstrations, and a training trajectory with matching grasp geometry supplies the motion template.

Let $p^\ell_t$ be the template position for arm $\ell$, $g^\ell$ its grasp target and $\hat g^\ell$ the scene-conditioned target. The position update is
\begin{equation}
 \tilde p^\ell_t=p^\ell_t+w_t(\hat g^\ell-g^\ell),\qquad
 w_t=3u_t^2-2u_t^3,\quad u_t=\min(t/t_g,1),
 \label{eq:robotwarp}
\end{equation}
where $t_g$ is the template grasp step. Gripper commands and orientations retain their template correspondence, with quaternion normalization at the interface. Sixteen target variants combine offsets of $\{0,0.02,0.04,0.06\}$ metres along the predicted tabletop direction and $\{-0.03,-0.01,0.01,0.03\}$ metres perpendicular to it. Both arms receive the same spatial offset, preserving the predicted relative grasp geometry. The untouched native future remains in the candidate set. Sequences retain the complete declared horizon, padding a shorter sequence with its final command when needed.

\textit{Relational preservation and execution.}
The grasp-timing readout compares the native bimanual future with the predicted grasp target. It finds the step of minimum mean left/right target distance among steps with both grippers closed; if none is closed, it uses the full sequence. Dividing that step index by the sequence length gives a temporal relation $q$. The timing gate selects the corrected future when $q<\tau$, otherwise preserving the native future. The alternative distance gate uses the minimum mean target distance and the opposite inequality. Threshold calibration uses paired development outcomes with preservation-oriented tie breaking. At execution, predicted geometry and candidate actions determine the gate decision. Candidate IDs are retained through subset construction, selection and execution. Matched simulator branches restore one initial snapshot and record executed length, first success and any execution error; the chosen 16-dimensional sequence is consumed by the same low-level execution interface.

\subsection{Autonomous driving: trajectory evidence and calibrated readouts}
\label{app:drivinginterface}
Drive-JEPA exposes 32 candidate trajectories, each containing eight ego-frame $(x,y,\theta)$ poses spaced by 0.5 seconds over a four-second horizon. The exporter captures a 256-dimensional proposal-query summary at the native scorer, the predicted total score, and the six native auxiliary output channels. These queries describe candidate evidence; recorded future camera frames and official outcome scores are stored separately. A tie-aware within-scene coordinate for native score $s_i$ is
\begin{equation}
 r_i=K^{-1}\!\left(\sum_j\mathbf1[s_i>s_j]
                   +\tfrac12\sum_j\mathbf1[s_i=s_j]\right).
\end{equation}
Unlike the lower-is-better goal costs used in manipulation, higher driving scores are preferred. Concatenating query, six auxiliary channels, score, rank and 24 trajectory coordinates gives the 288-dimensional score-correction input. The anchor-relative readout omits the six auxiliary channels and uses 282 dimensions. Normalization statistics come from fitting examples only.

\textit{Score correction.}
A linear projection, LayerNorm and GELU map each token to 64 dimensions. Two four-head Transformer layers with 128-dimensional feed-forward blocks, zero dropout and no candidate-position encoding compare the set. A zero-initialized head gives $\tilde s_i=s_i+0.2\tanh(w^\top h_i+b)$. Calibration scales this residual before selection. The loss combines squared error to official candidate quality, pairwise ordering near the native high-score boundary and a squared-correction penalty. The independent MLP control replaces set attention with two per-candidate layers; score fusion operates on native scalar outputs. The recipe uses AdamW, learning rate $10^{-3}$, weight decay $10^{-4}$, batch size 16, gradient clipping at 1 and 30 epochs; residual scale and epoch are chosen on calibration examples.

\textit{Anchor-relative correction and risk.}
With native winner $b=\argmax_i s_i$, the relative head gives $\delta_i=0.2\tanh(w^\top(h_i-h_b))$, so $\delta_b=0$ exactly. Two supervised logits predict responsibility-collision and TTC failure. Their sigmoid probabilities define $v_i=p_i^{\rm NC}+(5/12)p_i^{\rm TTC}$. The calibrated utility and switch rule are
\begin{equation}
 U_i=s_i+\alpha\delta_i-\lambda\max(v_i-v_b,0),\qquad
 a^*=\begin{cases}a_j,&U_j>U_b+m,\\a_b,&\text{otherwise},\end{cases}
 \quad j=\argmax_i U_i.
 \label{eq:driveswitch}
\end{equation}
Native auxiliary channels are not substituted for the supervised risk probabilities. The relative loss combines weighted smooth-$L_1$ regression of candidate benefit, pairwise benefit ordering and correction regularization; enabling risk adds weighted binary cross-entropy. Boundary and safety-negative weights concentrate learning on consequential changes. This configuration uses AdamW with learning rate $3\times10^{-4}$, weight decay $10^{-3}$, batch size 64 and gradient clipping at 1. Leave-one-source-log-out calibration selects epochs from $\{8,16,32\}$, $\alpha\in\{0,0.1,0.25,0.5,1\}$, $\lambda\in\{0,0.1,0.3\}$ and $m\in\{0,0.005,0.02\}$, retaining native fallback. Source-log segments stay together, and each fold computes its own fitting normalization. The final chosen candidate is passed unchanged to the official trajectory evaluator. PDMS and its components follow Appendix~\ref{app:metrics}.

\subsection{Physical robot: predicted futures to action sequences}
\label{app:physicalinterface}
The physical-robot interface keeps action-conditioned prediction, candidate selection and device control separate. Recorded windows supply observed frames, goal observations, candidate actions and persistent IDs. For each predictive geometry, the relational input concatenates normalized terminal-minus-goal descriptors with within-set ranks of native costs. A 64-dimensional projection, two four-head set-encoder layers and a rank-eight correction head produce bounded changes to the mean ordinal base score. At inference, the scoring function maps candidate futures and native costs to the selected action.

When candidate actions are generated by cross-entropy search, the native criterion compares candidates across search iterations; set-relative ranks are applied only after constructing the shared decision pool. This preserves score semantics as the search distribution changes. The selected candidate ID passes unchanged to the device-specific controller, which executes its original action sequence with the appropriate action conversion, timing and calibration. Predictive selection and low-level execution thus form explicit interface stages.

\section{Theoretical Properties of Decision Alignment}
\label{app:theory}
The following properties connect the decision-local prediction gap to bounded alignment, scale-free predictive evidence and native-distance realization. We use a fixed candidate set with $K\ge2$ and retain the score conventions of Section~\ref{sec:method}.

\subsection{Global prediction quality and execution decisions}
\label{app:globalgap}
Let $q_i$ denote realized execution cost and $c_i$ its prediction on the same scale, with smaller values preferred. The prediction selects $\hat\imath=\argmin_i c_i$; its execution regret is $q_{\hat\imath}-\min_iq_i$.

\begin{proposition}[Global agreement and decision error]
\label{prop:globalgap}
For any $K\ge2$ and $\Delta>0$, there exist distinct nonnegative realized costs and predictions in $[0,2\Delta]$ such that
\begin{equation}
 \frac{1}{K}\sum_{i=1}^K(c_i-q_i)^2=\frac{2\Delta^2}{K},
 \qquad \rho_{\mathrm{Sp}}(c,q)=1-\frac{12}{K(K^2-1)},
 \qquad q_{\hat\imath}-\min_iq_i=\Delta.
 \label{eq:globalgap}
\end{equation}
Thus vanishing candidate-average cost error and near-perfect global rank agreement can coexist with a fixed execution regret.
\end{proposition}
\begin{proof}
Set $q_1=0$ and $q_i=\Delta[1+(i-2)/(K-1)]$ for $i\ge2$. Set $c_1=q_2$, $c_2=q_1$, and $c_i=q_i$ otherwise. Only the first two prediction errors are nonzero, each with magnitude $\Delta$. Their ranks exchange positions, so the sum of squared rank differences is $2$; the Spearman formula yields Equation~\ref{eq:globalgap}. Prediction selects candidate $2$, whereas candidate $1$ minimizes realized cost, giving regret $\Delta$.
\end{proof}
For $K=63$, the correlation is approximately $0.999952$. Any success criterion $q_i\le\theta$ with $0<\theta<\Delta$ makes the predicted choice fail despite an available successful candidate. The construction isolates the distinction between average prediction quality and executed decisions; Figure~\ref{fig:mismatch} provides the corresponding empirical diagnosis.

\subsection{Decision margins and bounded alignment}
\label{app:marginproperties}
Uniform error control relates prediction to the realized decision margin. Here the error bound compares costs on a common scale.
\begin{lemma}[Execution margin under bounded cost error]
\label{lem:costmargin}
If $|c_i-q_i|\le\eta$ for all candidates, then $q_{\hat\imath}-q_{i^\star}\le2\eta$, where $i^\star\in\argmin_iq_i$. If $q_j-q_{i^\star}>2\eta$ for every $j\ne i^\star$, prediction selects $i^\star$ uniquely.
\end{lemma}
\begin{proof}
Since $c_{\hat\imath}\le c_{i^\star}$, we have $q_{\hat\imath}\le c_{\hat\imath}+\eta\le c_{i^\star}+\eta\le q_{i^\star}+2\eta$. Under the stated strict margin, every other candidate exceeds this bound.
\end{proof}
For relational alignment, the relevant margin is measured in the base score $b_i$ of Equation~\ref{eq:bounded}. In the dual-model instance, $b_i=\alpha r_i^L+(1-\alpha)r_i^T$ is calibrated ordinal fusion.
\begin{proposition}[Margin preservation and selection locality]
\label{prop:boundedalignment}
Let $s_i=b_i+\delta_i$, with $|\delta_i|\le\epsilon$. If $b_j-b_i>2\epsilon$, then $s_i<s_j$ for every admissible correction. Moreover, for $i_b\in\argmin_i b_i$ and $i_s\in\argmin_i s_i$,
\begin{equation}
 b_{i_s}-b_{i_b}\le2\epsilon.
 \label{eq:selectionlocality}
\end{equation}
\end{proposition}
\begin{proof}
The corrected difference satisfies $s_j-s_i\ge b_j-b_i-2\epsilon>0$. Also, $s_{i_s}\le s_{i_b}$ implies $b_{i_s}-b_{i_b}\le\delta_{i_b}-\delta_{i_s}\le2\epsilon$.
\end{proof}
\begin{corollary}[Base-winner preservation]
\label{cor:winnerpreservation}
If $b_j-b_{i_b}>2\epsilon$ for every $j\ne i_b$, bounded relational alignment retains $i_b$.
\end{corollary}
Consequently, the bounded selector can promote only candidates within $2\epsilon$ of the base minimum. The relational/base gate also satisfies this selection bound because it returns either $i_s$ or $i_b$. These properties apply to the bounded relational score in Equation~\ref{eq:bounded}; predictor adaptation and calibrated composition use the separate rules in Appendix~\ref{app:methoddetails}.

\subsection{Scale-free evidence across predictive geometries}
\label{app:ordinalinvariance}
\begin{proposition}[Monotone-scale invariance of ordinal evidence]
\label{prop:ordinalinvariance}
For a fixed candidate set, let $c_i^m$ be source $m$'s native costs, with ties resolved by persistent candidate IDs. Any strictly increasing transformation $f_m$ preserves every ordinal coordinate:
\begin{equation}
 \frac{\rank_{\mathcal A}(f_m(c_i^m))-1}{K-1}
 =\frac{\rank_{\mathcal A}(c_i^m)-1}{K-1}=r_i^m.
 \label{eq:ordinalinvariance}
\end{equation}
\end{proposition}
\begin{proof}
Strict increase preserves every strict comparison and equality among source costs. The same persistent IDs resolve the same ties, leaving each rank unchanged.
\end{proof}
Each predictive source may use its own transformation. This invariance concerns ordinal evidence, while dense descriptors retain their model-specific geometry. With descriptors and parameters fixed, the dual-model token and ordinal base score are unchanged, and so is their relational readout. Persistent IDs also make tie resolution independent of candidate presentation order, complementing the permutation-equivariant operator described in Appendix~\ref{app:methoddetails}.

\subsection{Exact realization through native latent distance}
\label{app:exactrealization}
\begin{proposition}[Exact native-distance realization]
\label{prop:exactrealization}
Let $\pi_i\in\{1,\ldots,K\}$ be the final strict ranks after gating. For a $D$-dimensional goal latent $z_g$, choose $\|u_i\|_{\rm RMS}=\sqrt{D^{-1}\|u_i\|_2^2}=1$. The realization $\tilde z_{i,H}=z_g+\pi_i u_i/(K+1)$ satisfies
\begin{equation}
 \|\tilde z_{i,H}-z_g\|_{\rm RMS}=\frac{\pi_i}{K+1},
 \qquad D^{-1}\|\tilde z_{i,H}-z_g\|_2^2=\left(\frac{\pi_i}{K+1}\right)^2.
 \label{eq:realizationidentity}
\end{equation}
Both native distances reproduce the full strict order and select $\argmin_i\pi_i$.
\end{proposition}
\begin{proof}
Subtracting $z_g$ and taking the RMS norm gives $\pi_i/(K+1)$ by homogeneity and unit normalization. Squaring yields the mean-squared identity. Both expressions are strictly increasing in the positive rank $\pi_i$, establishing order and selected-action identity.
\end{proof}
The original terminal direction supplies $u_i$ whenever its displacement is nonzero; a fixed unit-RMS direction supplies the zero-displacement case. Earlier future steps remain unchanged under Equation~\ref{eq:ordinal}. This property expresses learned decision structure through the native planning interface. Table~\ref{tab:deployment} verifies numerical recovery in the deployed implementation; Appendix~\ref{app:representation} records the implementation conventions and temporal-transport diagnostic.

\end{document}